\documentclass{article}

\usepackage[utf8]{inputenc} 
\usepackage[T1]{fontenc}    
\usepackage{hyperref}       
\usepackage{url}            
\usepackage{booktabs}       
\usepackage{amsfonts}       
\usepackage{nicefrac}       
\usepackage{microtype}      
\usepackage{graphicx}
\usepackage{amsmath}
\usepackage{systeme}
\usepackage{enumitem}
\usepackage{placeins}

\usepackage{tabularx}

\usepackage{xcolor}
\definecolor{darkgreen}{rgb}{0, 0.6666, 0}
\usepackage{wrapfig}

\def\s{S}
\def\i{I}
\def\d{D}
\def\a{A}
\def\r{R}
\def\t{T}
\def\h{H}
\def\e{E}

\def\th{d}
\def\eps{\varepsilon}

\def\GF{\textit{GF}}

\newif\ifdraft
\drafttrue
\draftfalse

\def\R{{\bf R}} 
                
\def\var{x}
\def\cvar{u}
\def\X{X}

\DeclareMathOperator{\vect}{vec}

\usepackage {amsthm}
\theoremstyle{plain}  
\newtheorem{proposition}{Proposition}

\usepackage{arxiv}
\newcommand*\samethanks[1][\value{footnote}]{\footnotemark[#1]}

\begin{document}
%
\title{An Optimal Control Approach to Learning in SIDARTHE Epidemic model}

\author{Andrea~Zugarini\thanks{Universities of Florence and Siena, Italy. \texttt{\{andrea.zugarini,enrico.meloni\}@unifi.it}.} \And
        Enrico~Meloni\samethanks \And
        Alessandro~Betti\thanks{University of Siena, Siena, Italy. \texttt{alessandro.betti2@unisi.it, andrea.panizza75@gmail.com}.}\And
        Andrea~Panizza\samethanks \And
        Marco~Corneli\thanks{Université C\^ote d'Azur Center of Modeling, Simulation \& Interaction, Nice, France and Inria, CNRS, Laboratoire J.A. Dieudonné, Maasai research team, Nice, France. \texttt{marco.corneli@univ-cotedazur.fr}.} \And
        Marco~Gori\thanks{University of Siena, Siena, Italy and Inria, CNRS, I3S, Maasai,  Université C\^ote d'Azur, C\^ote d'Azur, France. \texttt{marco@diism.unisi.it}.}%
}

\maketitle

\begin{abstract}
The COVID-19 outbreak has stimulated the interest in the proposal of novel epidemiological models to predict the course of the epidemic so as to help planning effective control strategies. In particular, in order to properly interpret the available data, it has become clear that one must go beyond most classic epidemiological models and consider models that, like the recently proposed SIDARTHE, offer a richer description of the stages of infection. The problem of learning the parameters of these models is of crucial importance especially when assuming that they are time-variant, which further enriches their effectiveness. In this paper we propose a general approach for learning time-variant parameters of dynamic compartmental models from epidemic data. 
We formulate the problem in terms of a functional risk that depends on the learning variables through the solutions of a dynamic system.
The resulting variational problem is then solved by using a gradient flow on a suitable, regularized functional. We forecast the epidemic evolution in Italy and France. Results indicate that the model provides reliable and challenging predictions over all available data as well as the fundamental role of the chosen strategy on the time-variant parameters. 
\end{abstract}

\section{Introduction}
The novel coronavirus that emerged in Wuhan, China, at the end of 2019, severe acute respiratory syndrome coronavirus 2 (SARS-CoV-2) \cite{chan_familial_2020}, quickly spread in China and then to the rest of the world. As of September 30th 2020, at least 215 countries have been impacted, with over 33 millions detected cases, and over 1 million deaths\footnote{\url{https://www.ecdc.europa.eu/en/geographical-distribution-2019-ncov-cases}}. Huge efforts are underway to contain the pandemic. In absence of specific vaccines or effective drugs against COVID-19, the disease caused by SARS-CoV-2, governments have resorted to non-pharmaceutical interventions to prevent its spread, such as social distancing, mask wearing, isolation of the infected and their contacts, and in many cases national lockdowns. 

In the meantime, many researchers have focused their efforts on analyzing and forecasting the spread of COVID-19 \cite{chinazzi_effect_2020,bertozzi2020challenges,Giordano2020,zou_epidemic_2020,ferretti_quantifying_2020}. Predicting the effect of interventions, the evolution of the size of the outbreak, or the expected date for peak of active cases, are all results of paramount importance, that help policy makers to take the best decisions in the face of uncertainty. In order to obtain these results, a widely used class of epidemiological models is that of \textit{compartmental models}, such as the classical Susceptible-Infectious-Recovered (SIR) \cite{kermack1927contribution} and the Susceptible-Exposed-Infectious-Recovered (SEIR) models \cite{hethcote2000mathematics}. Compartmental models partition the population in disjoint groups, and, under the assumption of a homogeneous and uniformly mixed population \cite{murray_mathematical_2002}, they model the dynamics of each group as a system of constant-coefficient nonlinear Ordinary Differential Equations (ODE). The SIR and SEIR models, as well as variants such as SIDR \cite{Anastass2020} and SEIRDC \cite{Lin2020}, have been widely used to model the COVID-19 pandemic \cite{chinazzi_effect_2020, li_substantial_2020, ferguson2020report}, fitting the model parameters to the available public data. The mathematical properties of these models, such as the existence of a \textit{threshold phenomenon}, the possibility to estimate  the final size of the epidemic, the maximum number of infectious individuals at a given time and so on, are well-known and described for example in \cite{murray_mathematical_2002, hethcote2000mathematics, britton2010stochastic}.

An issue with fitting the standard SIR and SEIR models to publicly available data is the existence of a large fraction of undetected but infectious cases. As discussed in \cite{ferretti_quantifying_2020} and \cite{Giordano2020}, these undetected infections can often go unrecognized due to mildness of symptoms or lack thereof, thus exposing a far greater portion of the population to the virus than it would otherwise occur. 

In order to face the transmission due to undetected cases, in \cite{Giordano2020}, the authors consider a new epidemiological model, SIDARTHE, which extends the classical SIR model by discriminating between detected and undetected cases of infection, and different severity of illness.
The complex dynamic of the model is well suited for forecasting multiple aspects of the infection spread, and it achieves very interesting performance on predicting the pandemic evolution in both the Italian and French territories.

Typically, all the compartmental models assume rate coefficients to be constant in time. However, this assumption yields quite poor approximations over large observation windows during an outbreak. Clearly, the diffusion of a virus depends on multiple aspects that can change over time. A striking example is the case of national lock-downs aimed at dramatically containing the spread of the disease. In \cite{Giordano2020}, this issue is dealt with by assigning piece-wise constant coefficients in correspondence of lockdown policies changes, while in \cite{dehning_inferring_2020}, the authors model the coefficients as constant values separated by three linear transitions. 
However, such solutions require either precise knowledge of when and how the scenario changes or at least fixing a priori the number of breakpoints. This becomes unfeasible when 
there are multiple interacting phenomena such as local lockdowns, virus mutations, variations of treatments, therapies or infection screening.

In this paper, we propose an approach for learning time variant parameters of dynamic compartmental models
and present an approach that nicely reflects the spirit of most machine learning algorithms.  We formulate the learning process within the framework of optimal control theory  \cite{kwakernak1972linear,lions_optimal_1971,becker2000adaptive}. 
Then we attack the problem of parameter estimation by using
a gradient flow algorithm that, throughout the paper, is referred to
as $GF$. The algorithm, which alternates steps of ODE solutions with
gradient estimation, is shown to be very effective thanks to an 
appropriate regularization of the model parameters which properly identifies their weight along the temporal window of simulation. 
In particular, we learn time-variant coefficients of SIDARTHE, but clearly the algorithm is suitable for any other compartmental model whenever supervised data is available.

This paper is organized as follows. After a brief review of SIDARTHE (Section~\ref{sec:the_model}), we introduce the proposed learning framework in Section~\ref{sec:estimation_of_sir_coefficients}, and report the experiments in
Section~\ref{sec:exp}. Finally, some conclusions are drawn in Section~\ref{sec:conclusions}.

\section{SIDARTHE}\label{sec:the_model}
In order to define the terminology and the notations that 
we will use in the remainder of the paper,
in this section, we give a brief review of SIDARTHE~\cite{Giordano2020}.
The model is a dynamical system described by eight ordinary 
differential equations in the variables
\begin{equation}\s(t),\,\i(t),\,\d(t),\,\a(t),\,\r(t),\,\t(t),\,\h(t),\,\e(t).\end{equation}
Each of these 
quantities represents the population of a different compartment 
of the model at a certain temporal instant $t$. In particular each temporal instant 
$t$ is mapped to:
\[
\begin{aligned}
&\s(t)=\hbox{\# {\it susceptible individuals}},\\
&\i(t)=\hbox{\# {\it asymptomatic infected} which are {\it undetected}},\\
&\d(t)=\hbox{\# {\it asymptomatic infected} which have been {\it detected}},\\
&\a(t)=\hbox{\# {\it symptomatic infected} which are {\it undetected}},\\
&\r(t)=\hbox{\# {\it symptomatic infected} which have been {\it detected}},\\
&\t(t)=\hbox{\# \vtop{
\hsize=17pc\noindent{\it acutely symptomatic infected}{\it detected},}}\\
&\h(t)=\hbox{\# {\it healed}},\\
&\e(t)=\hbox{\# {\it deceased}}.
\end{aligned}
\]
The problem is then formally defined in terms of the Cauchy problem for the following ODE system\footnote{In~\cite{Giordano2020} they choose $\phi=\chi\equiv0$.}
\begin{equation}\label{sidarthe:eq}\medmuskip -1mu
\begin{cases}
\dot\s(t)=-\s(t)\bigr(\alpha\i(t)+\beta\d(t)+\gamma\a(t)
           +\delta\r(t)\bigl);\\
\dot\i(t)=\s(t)\bigr(\alpha\i(t)+\beta\d(t)+\gamma\a(t)+\delta\r(t)\bigl)
           -(\eps+\zeta+\lambda)\i(t);\\
\dot\d(t)=\eps\i(t)-(\eta+\rho)\d(t);\\
\dot\a(t)=\zeta\i(t)-(\theta+\mu+\kappa+\phi)\a(t);\\
\dot\r(t)=\eta\d(t)+\theta\a(t)-(\nu+\xi+\chi)\r(t);\\
\dot\t(t)=\mu\a(t)+\nu\r(t)-(\sigma+\tau)\t(t);\\
\dot\h(t)=\lambda\i(t)+\rho\d(t)+\kappa\a(t)+\xi\r(t)+\sigma\t(t);\\
\dot\e(t)=\phi\a(t)+\chi\r(t)+\tau\t(t),
\end{cases}
\end{equation}
where
\begin{equation}
\alpha,\,\beta,\,\gamma,\,\delta,\,\eps,\,\zeta,\,\eta,\,\theta,\,\kappa,\,\lambda,\, \mu,
\,\nu,\,\xi,\,\rho,\,\sigma,\,\phi,\,\chi,\,\tau,
\label{params}
\end{equation}
are the rates that specify the velocity of the flows 
between the compartments of the model, with the initial conditions

\begin{equation}
(\s^0,\i^0,\d^0,\a^0,\r^0,\t^0,\h^0,\e^0)=:z_0.
\label{initial-cond}
\end{equation}



\begin{figure*}
\def\legbox{\vrule width10pt height6pt depth1pt}
\def\<#1>{$\lower 2pt\hbox{\includegraphics{./sidarthe-#1.mps}}$}
\hbox{\includegraphics{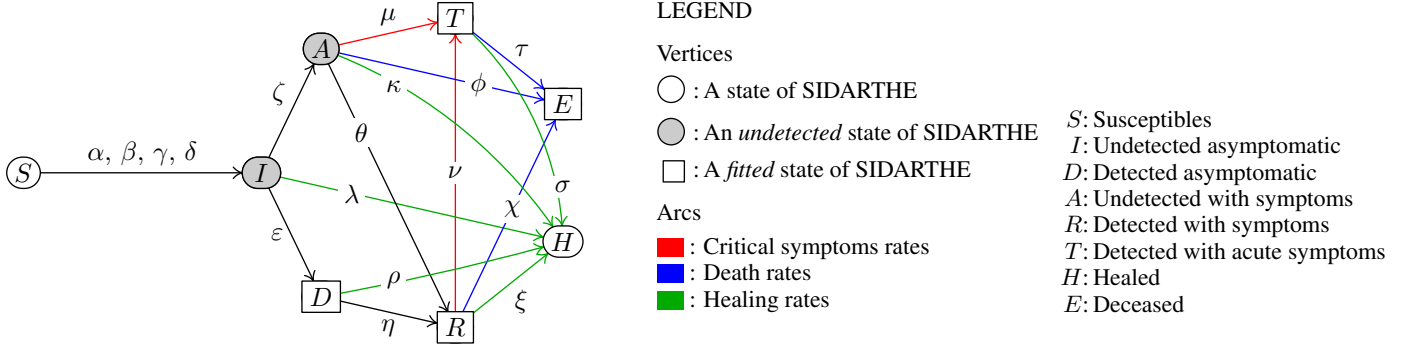}
\hskip 2pc
\vbox{\hsize=13pc\footnotesize LEGEND
\medskip
\par Vertices
\smallskip
\halign{\hfil#\,:\,& #\hfil\cr
\<4>& A state of SIDARTHE\cr
\noalign{\smallskip}
\<5>& An {\it undetected} state of SIDARTHE\cr
\noalign{\smallskip}
\<6>& A {\it fitted} state of SIDARTHE\cr
}
\medskip
\par Arcs
\smallskip
\par \textcolor{red}{\legbox} : Critical symptoms rates
\par \textcolor{blue}{\legbox} : Death rates 
\par \textcolor{darkgreen}{\legbox} : Healing rates
\vskip 1pc
}
\hskip -.5pc
\vbox{\hsize=13pc\footnotesize \phantom{LEGEND}
\medskip
\halign{\hfil$#$:\,& #\hfil\cr
\s& Susceptibles\cr
\i& Undetected asymptomatic\cr
\d& Detected asymptomatic\cr
\a& Undetected with symptoms\cr
\r& Detected with symptoms\cr
\t& Detected with acute symptoms\cr
\h& Healed\cr
\e& Deceased\cr
}
\vskip 1pc
}}
\caption{A DAG that shows the flow of a population through the compartments of the 
SIDARTHE model.}
\label{fig:graph}
\end{figure*}
In particular (see also Fig.~\ref{fig:graph}) we have that $\alpha$, $\beta$, $\gamma$ and
$\delta$ are the infection rates  between $\s$ and $\i$, $\d$, $\a$ and $\r$ respectively. Notice that these rates 
could be compared with the infection rate of the plain SIR model (the term in front of the bilinear term in the update rules of the susceptible
and the infected). The coefficients $\eps$ and $\theta$ govern the rate at which the asymptomatic and symptomatic undetected infected $\i$ and
$\a$ are detected, while $\zeta$ and $\eta$ are responsible for the transition between the asymptomatic and symptomatic classes (namely from $\i$ and 
$\d$ to $\a$ and $\r$). The quantities $\mu$ and $\nu$ control the flow from the symptomatic infected detected $\r$ and the symptomatic 
infected undetected $\a$ to the acutely symptomatic infected class $\t$ that, in turn, is connected to the set of deceased individuals $\e$
through the rate $\tau$. We also extend the SIDARTHE model
presented in~\cite{Giordano2020} with connections from $\a$ and $\r$ to $\e$, namely $\phi$ and $\chi$, to detect deceases outside Intensive Care Units (ICUs), as the ones occurred in elderly care facilities. Finally, $\kappa$, $\lambda$, $\xi$, $\rho$ and $\sigma$ represent the recovery rates.
Since the flows of the population through the eight compartments are directed (indeed the graph in Fig.~\ref{fig:graph} is a dag) all the rates must be non-negative.

The constants $\s^0$, $\i^0$, $\d^0$, $\a^0$, $\r^0$, $\t^0$, $\h^0$ and $\e^0$ in Eq.~\eqref{initial-cond} are assumed to be real non-negative
values and coupled with the SIDARTHE differential equations they specify a Cauchy problem.
Notice that if $\i^0=\d^0=\a^0=\r^0\equiv0$ then the infection cannot begin. 
From Eq.~\eqref{sidarthe:eq} it is also immediate to see that the total population is 
conserved since $\dot\s+\dot\i+\dot\d+\dot\a+\dot\r+\dot\t+\dot\h+\dot\e=0$.
As it is argued in~\cite{Giordano2020} an appropriate definition of 
the basic reproduction number in this model is 
\begin{equation}\label{eq:R_0}
\begin{split}
R_0:=\frac{1}{\eps+\xi}\biggl(\alpha&+
\frac{\beta\eps}{\eta+\rho}+\frac{\gamma\zeta}{\theta+\mu+\kappa+\phi}\\
&+\frac{\delta}{\nu+\xi+\chi}\biggl(\frac{\eta\eps}{\eta+\rho}
+\frac{\zeta\theta}{\theta+\mu+\kappa}\biggr)\biggr).
\end{split}\end{equation}
\autoref{eq:R_0} was appropriately modified to account for the inclusion of $\phi$ and $\chi$. In the SIDARTHE model, all the rates in Eq.~\eqref{params} are constant over time,
and are only changed in windows where different lockdown policies are defined. 
However, virus aggressiveness, social behavior, 
climate changes and different treatment of the disease, may all change during the development of the outbreak, motivating the extension of Eq.~\eqref{sidarthe:eq} to the case of truly time-variant coefficients.
In the next section we will discuss how it is possible to learn from data, in a meaningful way, 
the coefficients in Eq.~\eqref{params} as functions of time over the horizon $[0,T]$.


\section{Learning the SIDARTHE coefficients}\label{sec:estimation_of_sir_coefficients}
Let $\cvar\colon[0,T]\to\R^{18}$ be the map
\[
\begin{split}
u(t) = &(\alpha(t),\beta(t),\gamma(t),\delta(t),\eps(t),\zeta(t), \eta(t), \theta(t), \kappa(t),\\
&\quad\lambda(t),\mu(t), \nu(t), \xi(t), \rho(t), \sigma(t), \phi(t), \chi(t), \tau(t)),
\end{split}\] 
belonging to the functional space\footnote{Here we assume that $\X$ is Hilbert.} $\X$, 
and $z\colon[0,T]\to\R^8$ the vector valued function 
\[z(t):=(\s(t),\i(t),\d(t),\a(t),\r(t),\t(t),\h(t),\e(t)).\]
Let $\overline\d(\cdot,\cvar,z_0)$ be the solution for the variable $\d$ of Eq.~\eqref{sidarthe:eq} when 
the coefficients are the components of the function $u$, and the initial values of the 
compartments are specified by the values of $z_0\in\R^8$. In a similar manner 
let us also define $\overline\r$, $\overline\t$ and $\overline\e$ so that each of such quantities, considered as functions of
all their arguments, maps $[0,T]\times\X\times\R^8\to\R$. Lastly let
\begin{equation}\label{eq:recovered-diagnosed}
\begin{split}
\overline\h_d(t,\cvar,z_0):=\int_0^t &\rho(s)\overline\d(s,\cvar,z_0)+\xi(s)\overline\r(s,\cvar,z_0)\\ &\quad+\sigma(s)
\overline\t(s,\cvar,z_0)\, ds,\end{split}\end{equation}
which, roughly speaking, represents the number of diagnosed individuals who recovered when we initialize
 Eq.~\eqref{sidarthe:eq} with $z_0$ and for a given choice $u$ of the various rates.
 
The quantities $\overline\d$, $\overline\r$, $\overline\t$, $\overline\h_d$ and $\overline\e$
are the basic ingredients to define the risk that we will use to define the learning 
task. Indeed let us define $\varphi\colon [0,T]\times\X\to\R$ the following quadratic error
 
 \[\medmuskip-1mu
 \begin{split}
 \varphi(t,\cvar):=&\frac{e_\d}{2}(\overline\d(t,u,z_0)-\hat\d(t))^2+\frac{e_\r}{2}(\overline\r(t,u,z_0)-\hat\r(t))^2\\
&\, +\frac{e_\t}{2}(\overline\t(t,u,z_0)-\hat\t(t))^2+
 \frac{e_\h}{2}(\overline\h_d(t,u,z_0)-\hat\h(t))^2\\
 &\quad+\frac{e_\e}{2}(\overline\e(t,u,z_0)-\hat\e(t))^2,
 \end{split}\]
 where $\hat\d$, $\hat\r$, $\hat\t$, $\hat\h$ and $\hat\e$ are the observed time series and 
 $e_\d$, $e_\r$, $e_\t$, $e_\h$ and $e_\e$ are positive constants. 
 
Let $F\colon\X\to\R$ be\footnote{An appropriate choice for the functional 
space $\X$ in this case
could be $\X=H^1([0,T];\R^{18})$.}
\begin{equation}\label{eq:functional}
F(\cvar):=\int_0^T\frac{m}{2}|\dot\cvar(t)|^2+\varphi(t,u)\, dt,
\end{equation}
with $m>0$. Notice that this is an integral of a Lagrangian that is non-local in time since 
$\varphi$ depends on the whole trajectory of the variables $\cvar$ and not just
on the their values at time $t$.
Then, the learning of $\cvar$ corresponds to the following optimization problem 
\begin{equation}\label{eq:min-problem}
\min_{u\in \X} F(u).
\end{equation}
This problem resembles identification and optimal control problems that are
associated with the minimization of $F$, that can be tackled by means of the theory of Lagrange's multipliers~\cite{kwakernak1972linear}.
In this case the states $z$ are promoted to variables of the problem,
so that the quadratic error $\varphi$ can be written directly in terms of the components of $z$;
for example the term $(\overline\d(t,u)-\hat\d)^2\to(z_3-\hat\d)^2$. The minimization problem then 
is solved under the constrained dynamic of $z$ given by the SIDARTHE system of the 
form $\dot z(t)=\Phi(z(t),u(t))$ ($\Phi$ here can be deduced by the right-hand-side 
of the differential equation in~\eqref{sidarthe:eq}). Then the problem
Eq.~\eqref{eq:min-problem} can be recast into the following form:
\begin{equation}\label{eq:min-2}
\medmuskip-1mu
\begin{aligned}
\min_{\substack{\cvar\in\X\\ z\in Y}} &\quad
\begin{aligned}\int_0^T&\frac{m}{2}|\dot\cvar|^2+
\frac{e_\d}{2}(z_3-\hat\d)^2+\frac{e_\r}{2}(z_5-\hat\r)^2\\
&\quad+\frac{e_\t}{2}(z_6-\hat\t)^2+
 \frac{e_\h}{2}(\h_d(\cdot,z,u)-\hat\h)^2;\end{aligned}\\
 \hbox{subject to}&\quad\; \dot z=\Phi(z,u),
\end{aligned}
\end{equation}
where $Y$ is an appropriate functional space that contains 
functions $z$ satisfying the initial condition $z(0)=z_0$ and
\begin{equation}\label{eq:non-local-term}
\h_d(t,z,u):=\int_0^t \cvar_{14}z_3+\cvar_{13}z_5+\cvar_{15} z_6.
\end{equation}
Following this
approach, the solution is usually achieved by imposing the  stationarity condition
on~\eqref{eq:min-2},
which yields the Euler-Lagrange differential equations with appropriate
boundary conditions over the temporal variable $t$. In this problem,
however, the presence of the additional non-locality due to~\eqref{eq:non-local-term}
that persists also in the reformulation~\eqref{eq:min-2}
requires, for instance, to regard $\h_d$ as another variable and to add 
the differential equation for $\dot\h_d$, that can be 
readily be inferred from Eq~\eqref{eq:non-local-term}, to the constraints
$\dot z=\Phi(z,u)$.

While the parameter estimation based on Eq.~\eqref{eq:min-2} constitutes by itself a very interesting and promising research direction,
in this paper, we propose to pursue the minimization of~\eqref{eq:functional}
through a more direct approach, i.e. by approximating the
gradient flow $\cvar'=-\nabla F$ (see~\cite{ambrosio2008gradient}) by an 
explicit method that updates the trajectories $t\mapsto u(t)$ starting from 
a fixed initial configuration $u^0\in\X$. For this reason we are referring to 
the learning approach proposed in this paper as \GF\/ (Gradient Flow).
In practice the proposed learning algorithm is an implementation of the
following update rule
\begin{equation}
\cvar^{k+1}=\cvar^k-\nabla F(u^k),\quad k\ge0,\label{eq:grad-flow-explicit}
\end{equation}
where $\nabla F$ is (when it exists) the Fr\'echet derivative (see~\cite{ambrosetti1995primer}) 
of $F$ and
$\cvar^0\in\X$ is assigned. The term $|\dot u|^2/2$ in Eq.~\eqref{eq:functional}
is extremely important for the well-posedness of the learning
problem: the minimization of the mean quadratic loss alone could in principle
lead to highly irregular solutions  that have a low degree of generalization
power. Moreover the term $\Vert\dot u\Vert_{L^2}$ gives coerciveness to the
whole functional making it more suitable to be the objective of a minimization problem. This term yields a parsimonious solution where abrupt changes
are penalized.
Due to the presence of $|\dot u|^2$, the stationarity condition on the functional in 
Eq.~\eqref{eq:functional} also suggests that the derivatives of stationary 
points of $F$ on the boundaries $t=0$ and $t=T$ must be vanishing, thus 
offering an interesting consistency check for the numerical solutions that
we find. Indeed, we verified experimentally that this condition generally holds true on the learned parameters.


Before going on to the description of the algorithm in terms of a time discrete version of~\eqref{eq:grad-flow-explicit} which is machine implementable, 
we notice that we can softly enforce the positivity of the parameters 
by adding to the functional $F$ the term $e_P\int_0^T 1_{\{u(t)<0\}}(t)\, dt$,
where $1_A$ is the indicator function of the set $A$ and $e_P$ is a positive
constant.

\paragraph{Algorithmic details}
Consider a uniform partition $0=t_0<t_1<\cdots<t_N=T$ of the interval $[0,T]$
where $|t_{i+1}-t_i|=:\Delta t$ for all $i=0,\dots,N-1$.
Let  $f\colon\R^{18(N+1)}\to[0,+\infty)$ that maps $x\mapsto f(x)$ to be the ``discretized''
version of the functional $F$ where a point $x\in\R^{18(N+1)}$ in its domain can be tought 
as the concatenation of all the parameters sampled at the time grid defined above. 
The discrete counterpart of~\eqref{eq:grad-flow-explicit} is a classical gradient descent method
which starts from the value
$\var^0$ (whose components are the sampling of $u^0$ on the various $t_j$) 
and compute the vector sequence $\var^1, \var^2,\dots$ according to the update rule
\begin{equation}\label{eq:grad-desc}
\var^{k+1}=\var^k-\pi \nabla f(\var^k),
\end{equation}
where $\pi>0$ is the learning rate, and where now $\nabla f$ is the 
ordinary gradient in $\R^{18(N+1)}$ which can be computed at each step once
we choose a numerical solver for Eq.~\eqref{sidarthe:eq}. After that the 
SIDARTHE equations are numerically integrated, the non-local term $\varphi$
becomes simply a function of the variable $\var\in\R^{18(N+1)}$.
We found that the update rule~\eqref{eq:grad-desc} that defines the
gradient flow suffers a  normalization problem  which affects
the parameters at different time. Basically, since
the term $\varphi$ in Eq.~\eqref{eq:functional}  depends on 
the variables $\var$ though 
a numerical integration of~\eqref{sidarthe:eq}, changes in early (in time) parameters
result in greater variations of $\varphi$ than changes in
later (in time) parameters, since the latter will only affect the solution of the 
SIDARTHE in the last part of the interval $[0,T]$, 
whereas the former parameters contribute to modify the potential
$\varphi$ on most of the interval. This suggest that the components 
of $\nabla f$ that corresponds to $t_i$ close to $T$ are negligible
with respect to the same quantity evaluated at earlier times. This
makes the learning process either extremely unstable or exceedingly slow.  
In order to overcome this 
problem we modified the update rule~\eqref{eq:grad-desc} by introducing a 
regularization that is conceived for propagating the gradients from earlier
to later times:
\begin{equation}\label{eq:reg}
\medmuskip-1mu\mskip-2mu
\hat\var_i^{k+1}\mskip-4mu=
\begin{cases}
\hat\var_i^k-\pi_0(\nabla\hat f(\hat\var^k))_0& \hbox{if\ } i=0;\\
\hat\var_i^k-\pi_i(\nabla\hat f(\hat\var^k))_i+\omega_i(\hat\var_{i-1}^{k+1}-\hat\var_{i-1}^{k}) 
\mskip -5mu& 
\hbox{if\ }i>0,
\end{cases}\end{equation}
where $\hat x_i\in\R^{18}$ for $i=0,\dots,N$ are the slices of $x$ that correspond to the value of the
parameters at time $t_i$. The function $\hat f$ is defined accordingly (for further details
see Appendix~\ref{appendix}).
We choose
\begin{equation}\label{eq:update_rule_params}
\pi_i\equiv\pi(t_i):=\frac{\pi_0}{1+at_i},\quad \omega_i\equiv\omega(t_i):=\frac{1}{1+e^{-bt_i}},
\end{equation}
where $\pi_0$, $a$ and $b$ positive parameters.

Note that the above scheme makes sense when we pass to the 
continuous limit with respect to the index $i$. For this reason the following proposition
is of interest
\begin{proposition}
If the solutions of the gradient flow $u'=-\nabla F(u)$ are continuous function of 
time, then Eq.~\eqref{eq:reg} is a discrete approximation of the following update rule for
$\cvar$:
\[
u^{k+1}(t)=u^k(t)-\frac{\pi(t)}{1-\omega(t)}(\nabla F(u^k))(t).
\]
\end{proposition}
\begin{proof}
It is sufficient to notice that, since we look for a solution which is continuous in $t$ in the limit 
$\Delta t\to 0$ we must have, for each $k\ge0$, $\vert\hat\var^k_{i-1}-\hat\var^k_{i}\vert\to 0$.
Then Eq.~\eqref{eq:reg} in the continuous limit becomes 
\[(1-\omega(t))(u^{k+1}(t)-u^k(t))=-\pi(t)(\nabla F(u^k))(t),\]
which is exactly what we wanted to prove.
\end{proof}
This proposition shows that the update scheme defined in
Eq.~\eqref{eq:reg} is basically equivalent to introducing an increasing, time-dependent learning rate.
Notice that, the term proportional to $\omega$ in Eq.~\eqref{eq:reg}
is reminiscent of the classic momentum term \cite{mcclelland1986parallel}. However, 
it also involves relations in the temporal domain which are neglected in other frameworks. 
Due to this similarity, 
in what follows, we refer to this term as the
\textit{temporal momentum}.

Function $\pi(t)/(1-\omega(t))$ is, with appropriate choices of the parameters $a$ and $b$, a monotonically
increasing function for $t>0$ and in particular for large $t$ we have
$
\pi(t)/(1-\omega(t))\simeq \pi_0 e^{b t}/at 
$.

\section{Experiments}\label{sec:exp}
\def\|#1|{\includegraphics[width=0.32\textwidth]{#1}}
\begin{figure*}
\centering
\|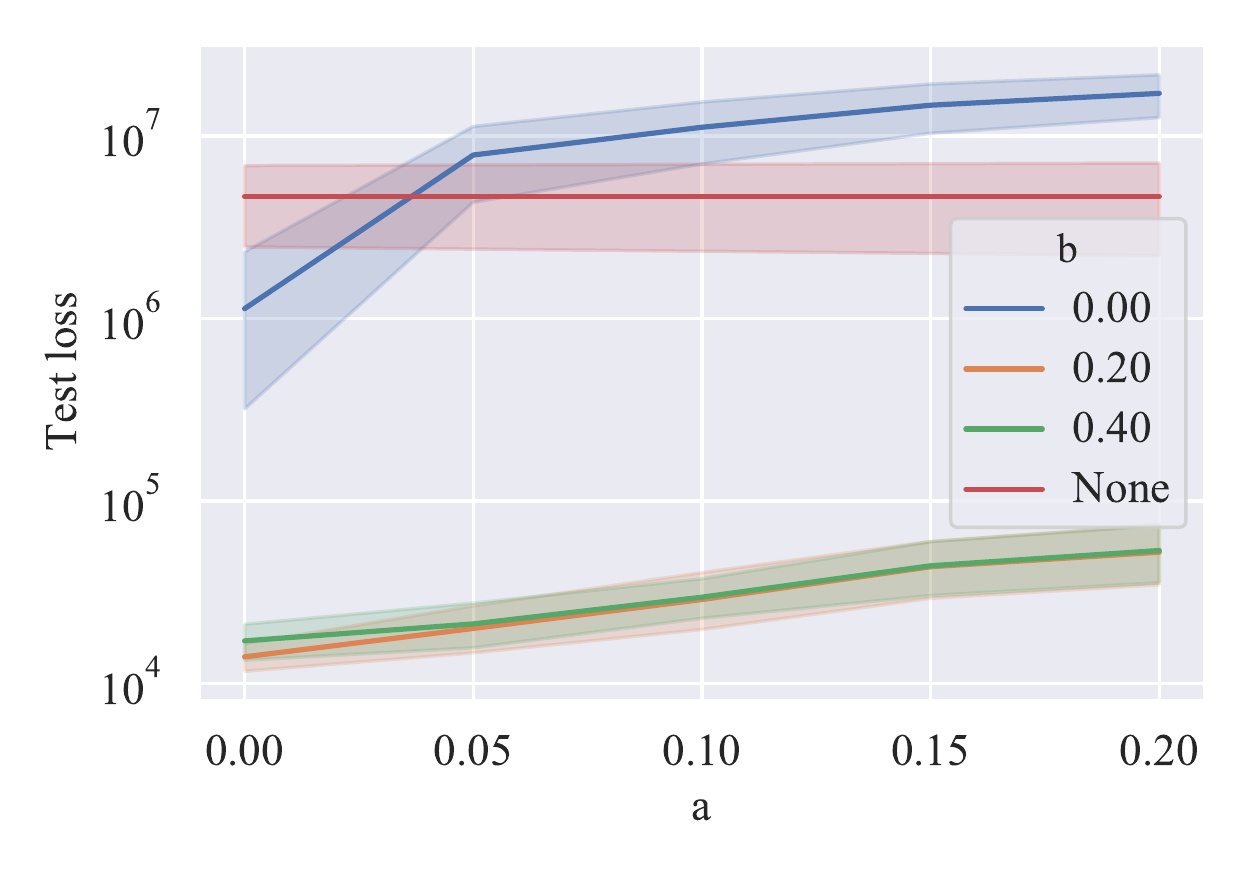|
\hfil
\|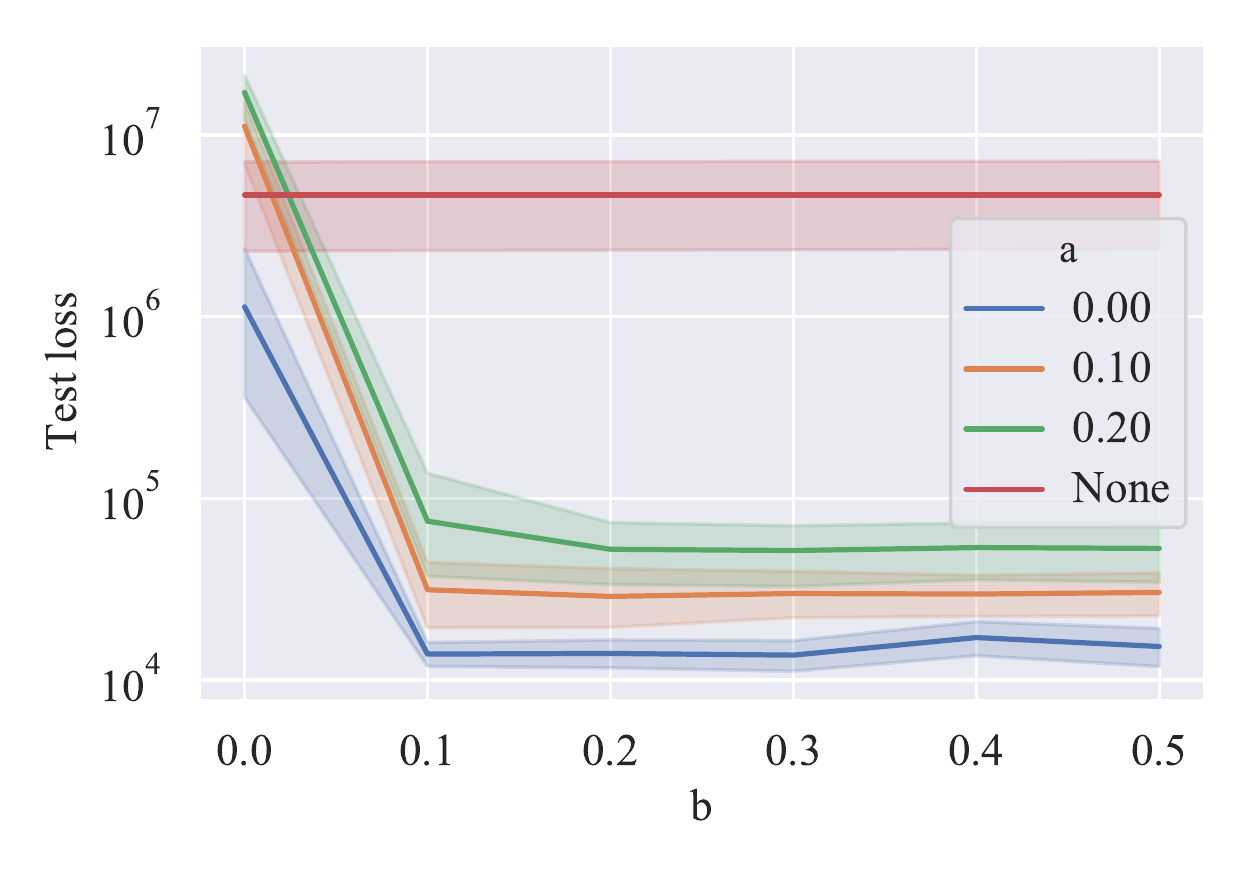|
\hfil
\|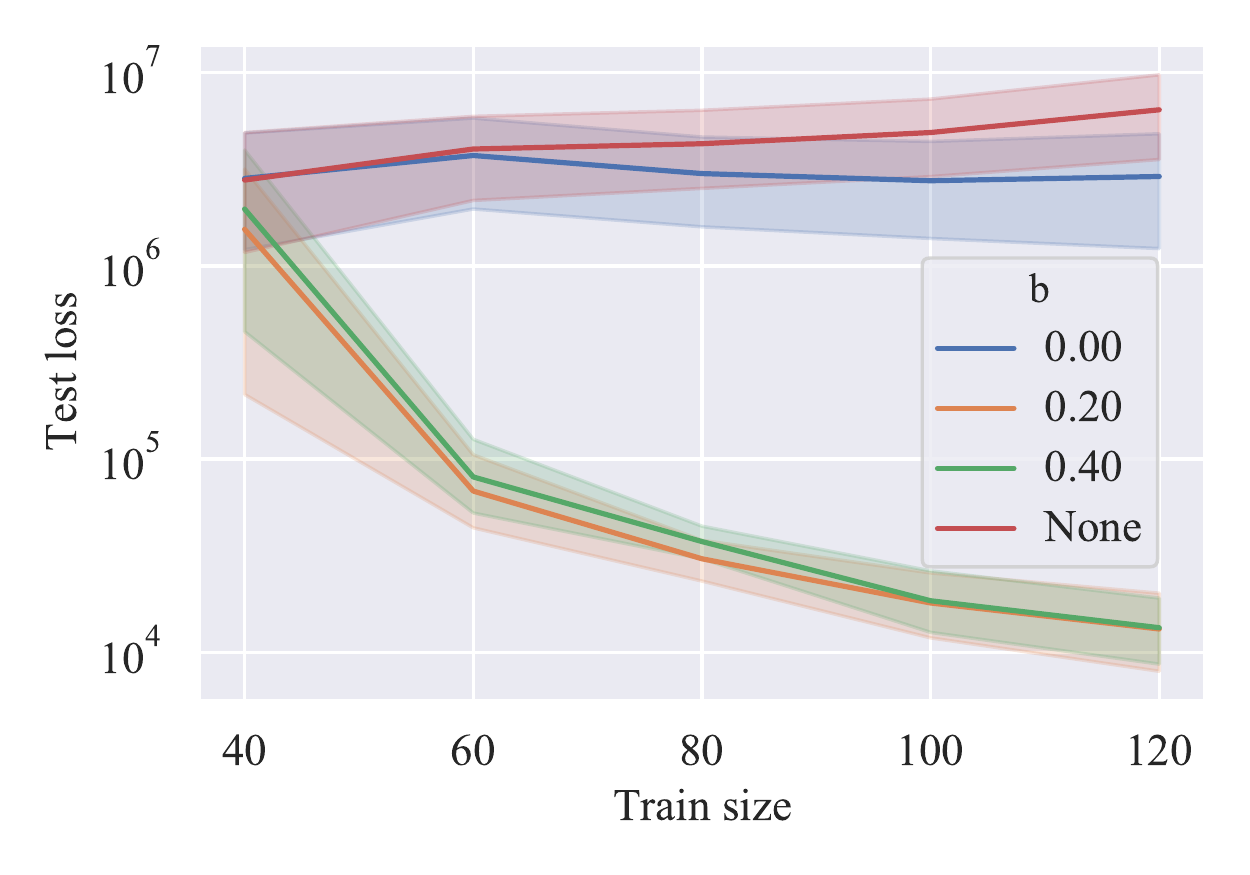|
\caption{Test loss values for different values of $a$, $b$ and $T$. Red line is the baseline where \textit{temporal momentum} is disabled. Values are reported with 95\% confidence intervals.}
\label{fig:ablation_momentum}
\end{figure*}

The analysis of our learning framework is carried out on the Italian\footnote{\url{https://github.com/pcm-dpc/COVID-19/tree/master/}} and French\footnote{\url{https://github.com/opencovid19-fr/data}} epidemiological data, gathered from official daily reports up to September 30, 2020.
This section is divided in two parts. First, we discuss the results of the ablation study, confirming  the importance of both the regularization term (Eq.~\eqref{eq:functional}) and the update rule (Eq.~\eqref{eq:reg}) for the learning process.
Then, we fit SIDARTHE on the Italian and French data. The code to reproduce all the experiments is available online\footnote{\url{https://github.com/sailab-code/learning-sidarthe}}.
The differential equations were solved by Heun's method \cite{quarteroni_numerical_2007}, implemented in PyTorch \cite{paszke2017automatic}. The automatic differentiation in PyTorch computes the gradient $\nabla f$ for each time-variant parameter.

\subsection{Ablation study}\label{subsec:ablation}

The learning of the SIDARTHE rates is performed via the update rule in Eq.~\eqref{eq:reg}, under the constraint on the first order derivative $\dot\cvar(t)$ introduced in Eq.~\eqref{eq:functional}. 
The impact of these two components (update rule and first order constraint) on the learning process depends on the values of the hyper-parameters $\{a,b\}$ in Eq.~\eqref{eq:update_rule_params} and $m$  in Eq.~\eqref{eq:functional}, respectively.
The aim of this section is to discuss and quantify the role of these hyper-parameters.

The Italian data set alone is considered in this section. Each experiment is repeated $20$ times, provided with a random initialization $x^0$ of the model parameters. Unless specified differently,
the training data set counts 120 \textit{consecutive} data points and the subsequent 20 samples are used for test. 


\paragraph{Temporal momentum.}
We performed a grid search on the hyper-parameter space of $\{a,b\}$. For $a$ we considered 5 equally spaced values in the interval $[0, 0.2]$. For $b$ we considered 6 equally spaced values in the interval $[0, 0.5]$. For each pair $\{a,b\}$ we trained 20 models. Additionally, we trained other 20 models where \textit{temporal momentum} (henceforth, \textit{momentum}) was disabled (i.e. $\omega_i=0$ for all $i$, thus reducing to a standard gradient descent). In total, we trained $5\times6\times20 + 20 = 620$ different models. The results are presented in Fig.~\ref{fig:ablation_momentum}.
The plots clearly show that the \textit{momentum} term improves the stability of the learning process. In particular, we see that the improvement saturates for $b > 0.1$. Conversely, $a > 0$ deteriorates the performances. 
Since the y axis is plotted with logarithmic scale, the confidence intervals are even narrower for $b > 0.1$ and $a = 0$.
\def\|#1|{\includegraphics[width=0.32\textwidth]{#1}}
\begin{figure*}
\centering
\|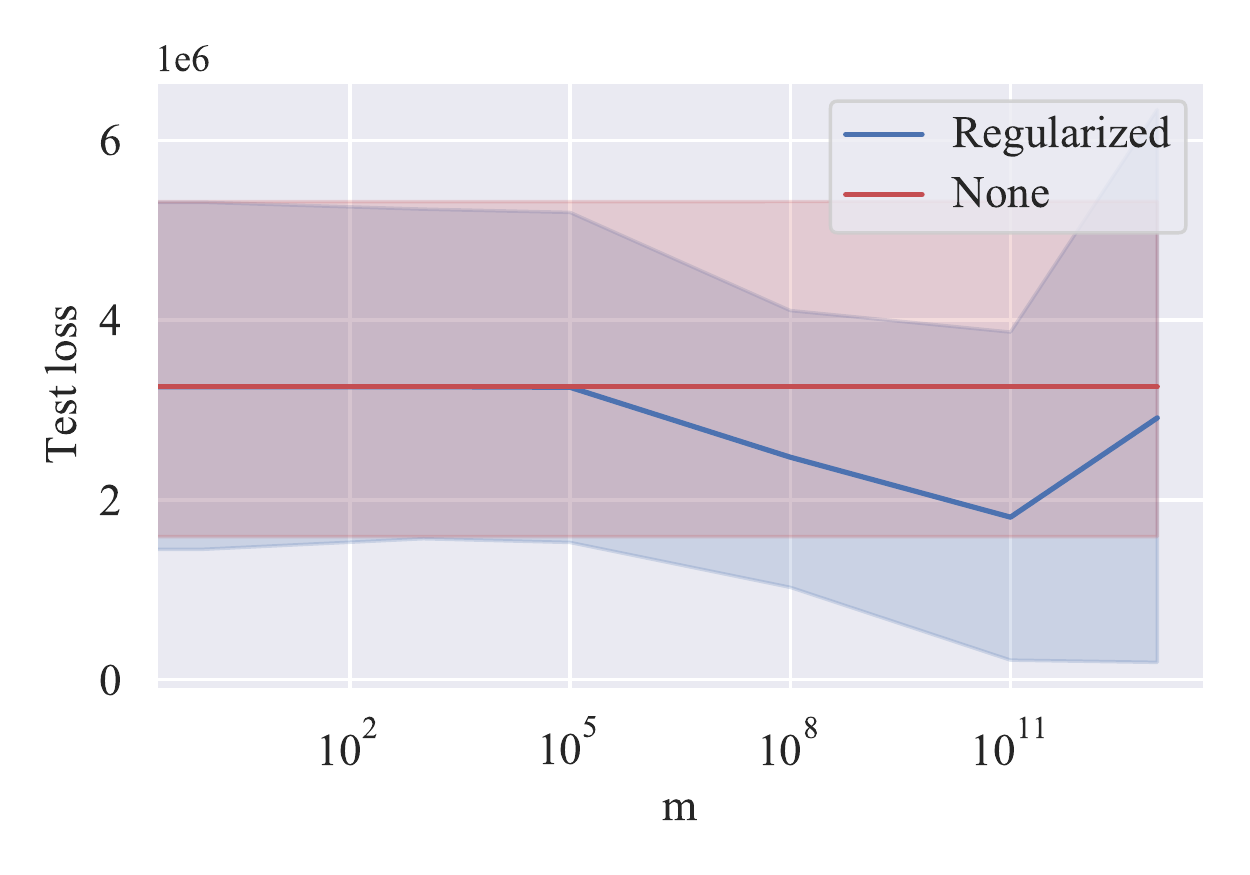|
\hfil
\|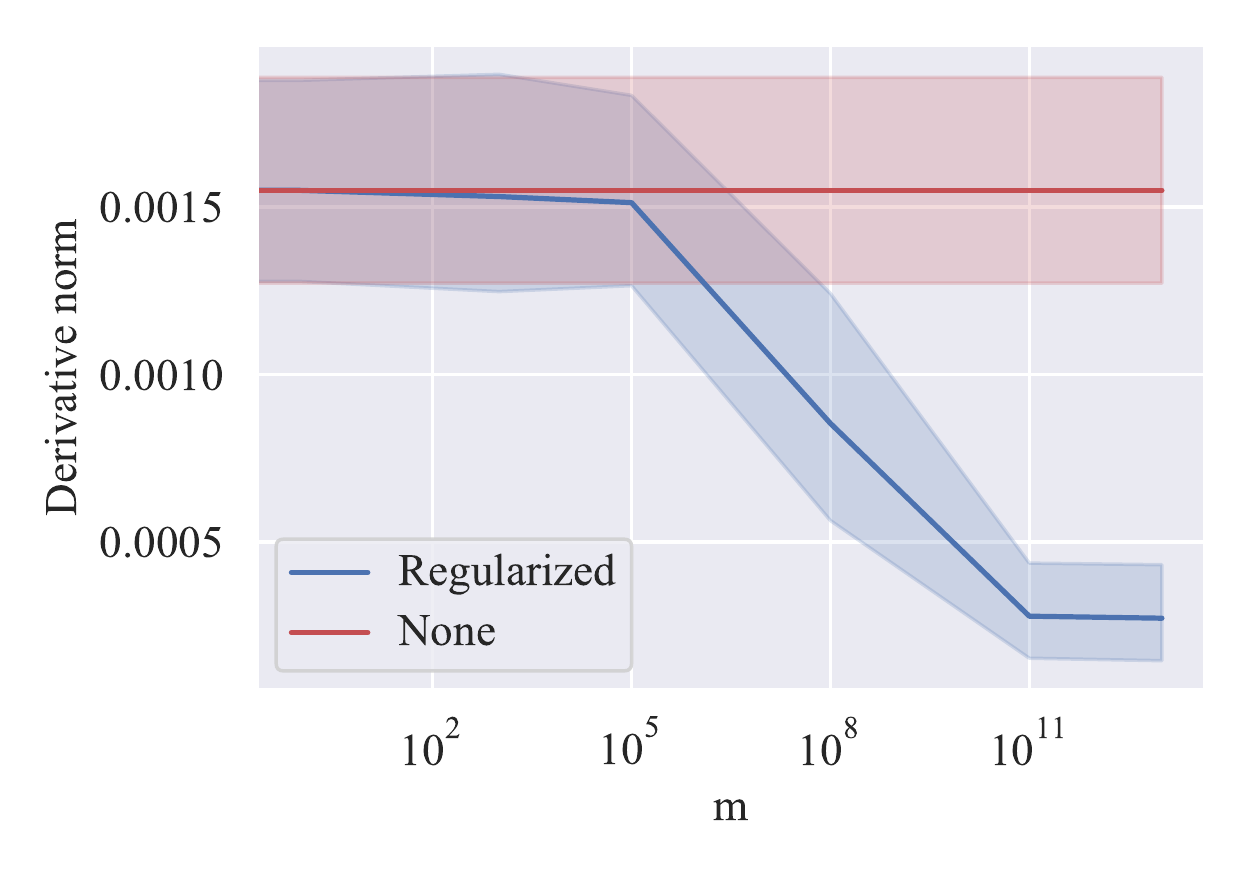|
\caption{Test loss values for different values of $a$, $b$ and $T$. Red line is the baseline where \textit{momentum} is disabled. Values are reported with 95\% confidence intervals.}
\label{fig:ablation_derivative}
\end{figure*}

We then performed a second grid search on the hyper-parameter space of $\{b,T\}$, where the values of $b$ are the same as described above, and we considered 5 equally spaced values of $T$ in $[40,120]$. In this case too, for each value of the pair $\{b,T\}$, 20 models were trained. In addition, for each value of $T$, we trained additional $
20$ models with \textit{momentum} disabled. In this setting, a total of $5\times6\times20 + 5\times20 = 700$ models were trained. Results are presented in Fig.~\autoref{fig:mmt_t}. The plot shows that when the \textit{momentum} term is disabled, the model performs poorly on test, and the learning has wider confidence intervals. Instead, when \textit{momentum} is enabled with a high enough value for $b$, the test loss becomes lower and the confidence intervals significantly narrow down.
These experiments show that the \textit{momentum} term dramatically improves the learning process, by further minimizing the (test) loss function and also reducing the dependency on the initial value $u_0$.  


\paragraph{Regularization.} To evaluate the effectiveness of the \textit{derivative} term, we performed a grid search on the hyper-parameter space of $m\in\{0, 1., 10^3, 10^5, 10^8, 10^{11}, 10^{13}\}$. For each value of $m$ we trained 20 models, for a total of $7\times20 = 140$ trained models. We plot the test loss as function of the weight, as shown in Fig.~\autoref{fig:ablation_derivative}.
The results show that, except for $m = 10^{11}$, the derivative term is not significantly changing the test loss. Instead, we see that the norm of the derivative of the parameters steadily decreases for $m > 10^5$. This means that the derivative term contributes to enforcing parameters as smoother functions of time, without significantly degrading the generalization of the learning.

\subsection{Outbreak Forecasting}




%
We forecast the epidemic spreading in Italy and France. 
We trained our models in the time span going from February, the 24th, to August, the 30th, i.e. overall 188 days. 
The following 31 days were used for validation and test.
In particular, we considered the period August, the 31st up to September, the 6th, for validation (7 days) and September, the 7th, up to September, the 30th, for test (25 days). 
The fitting was performed on the time series appearing in the functional risk $F$ in Eq.\eqref{eq:functional}. i.e. $\hat\d,\hat\r,\hat\t,\hat\h,\hat\e$. These values are all available from the Italian reports, whereas in the French official data, only $\hat\r,\hat\t,\hat\e$ are explicitly observed, along with the \textit{cumulative number of infectious} and the number of \textit{hospitalized individuals that recovered}, defined here as $C_I(t)$ and $\h_h(t)$, respectively. Instead, hospitalized infected individuals corresponding to $\d$ (i.e. the proxy of the asymptomatic detected people) are not traced. Consequently, we did not have direct information about their number and recovery date.  To extract $\hat\d$ and $\hat\h$ we made the assumption that asymptomatic individuals heal after a period $d$, that was set to $14$ days, i.e. the quarantine period commonly established by national governments. In such a way, active asymptomatic infectious $\hat\d$ and recovered individuals $\hat\h$ were estimated (at time $t$) as follows:
\[
\begin{aligned}
&\hat\d(t) = C_I(t) - \hat\t(t) - \hat\r(t) - \hat\d(t-d)\\
&\hat\h(t) = \h_h(t) + \hat\d(t-d)
\end{aligned}
\]
Moreover, some daily French reports have partial or total missing information, causing the presence of many missing data. Due to the rich presence of noise and missing data in the early stages of French outbreak, the model fitting begins from March, the 17th instead of February, the 24th, while validation and test dates were left unchanged. The remaining missing targets within training/validation/test periods were simply ignored for learning and evaluation.

\def\|#1|{\includegraphics[width=0.33\textwidth]{#1}}
\begin{figure*}[!hb]
\hbox to\hsize{\|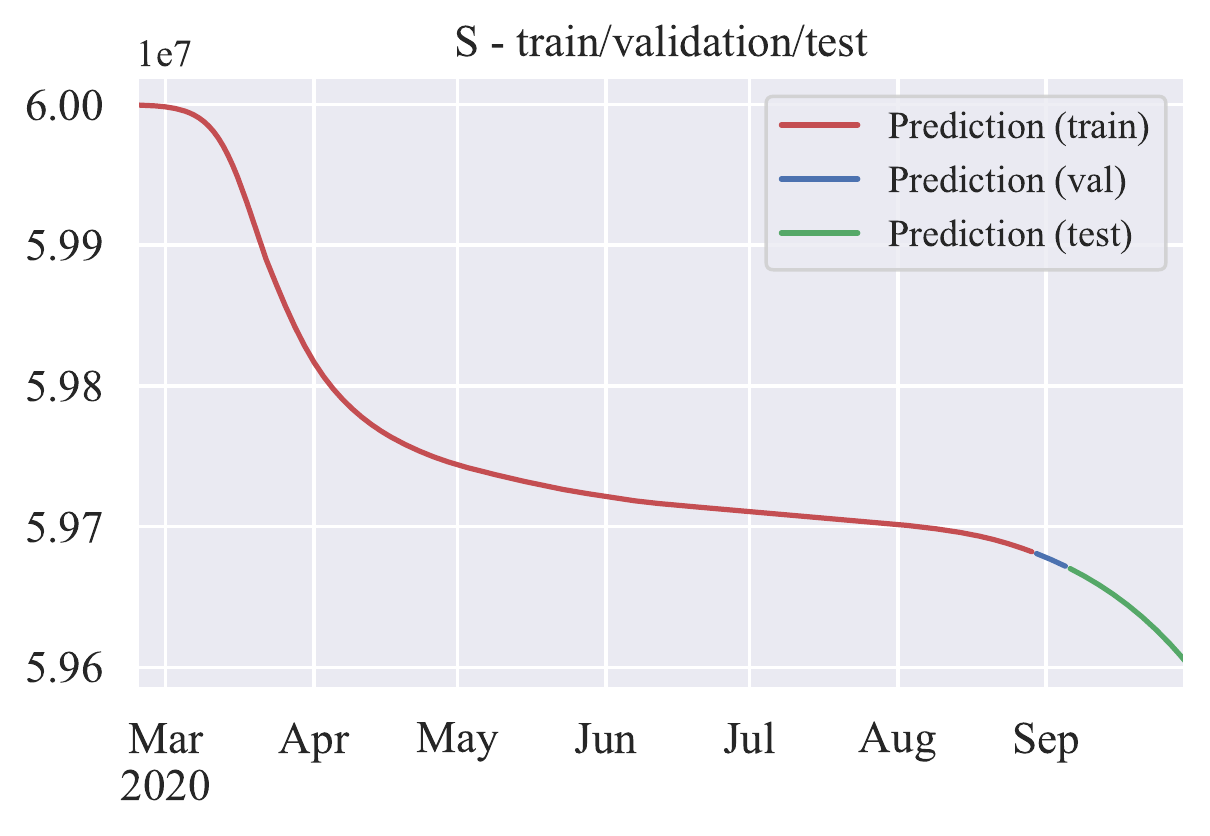|\hfil\|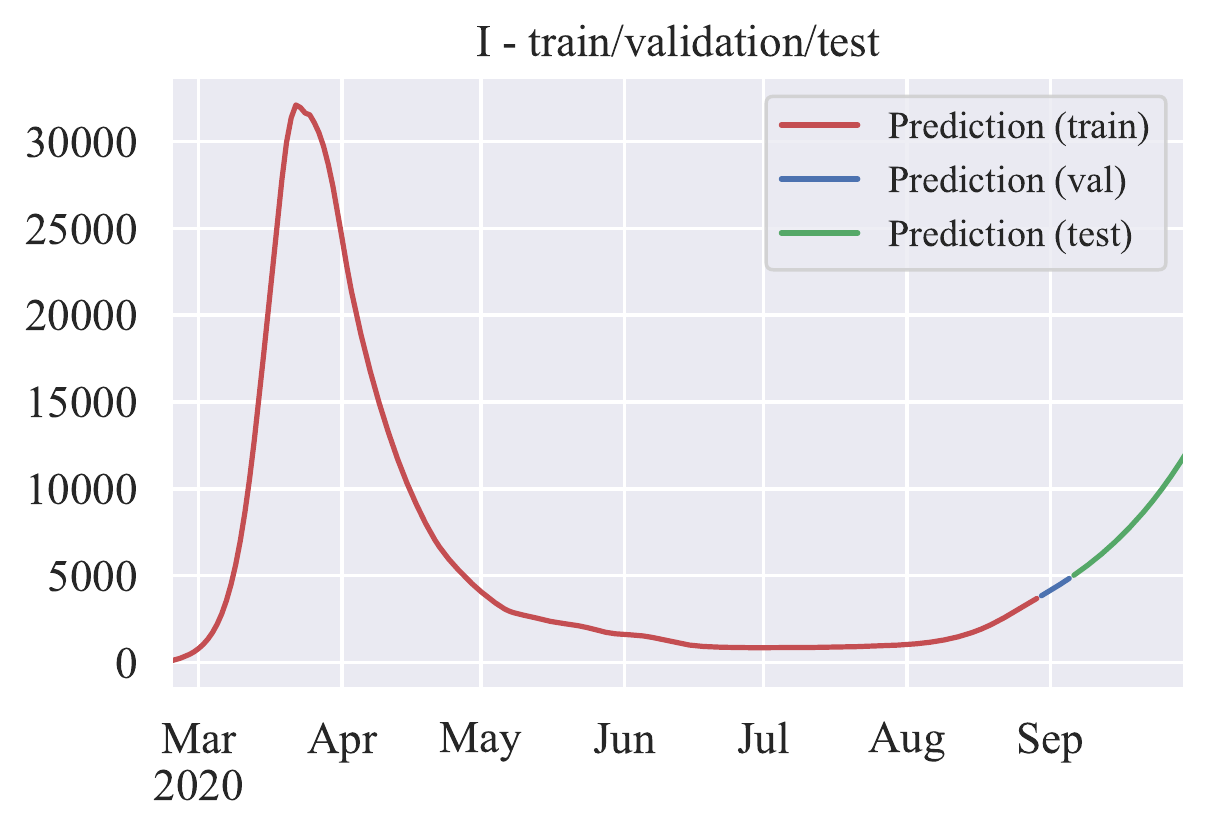|\hfil\|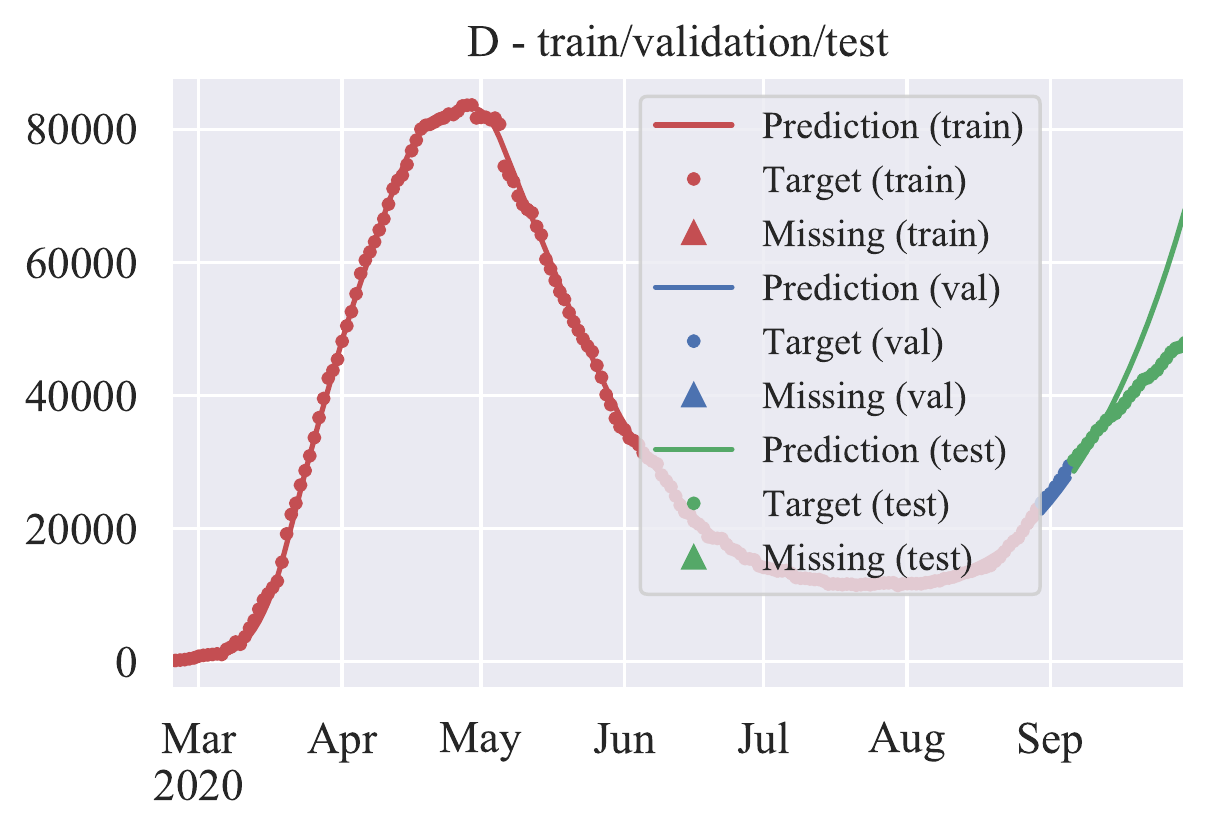|}
\hbox to\hsize{\|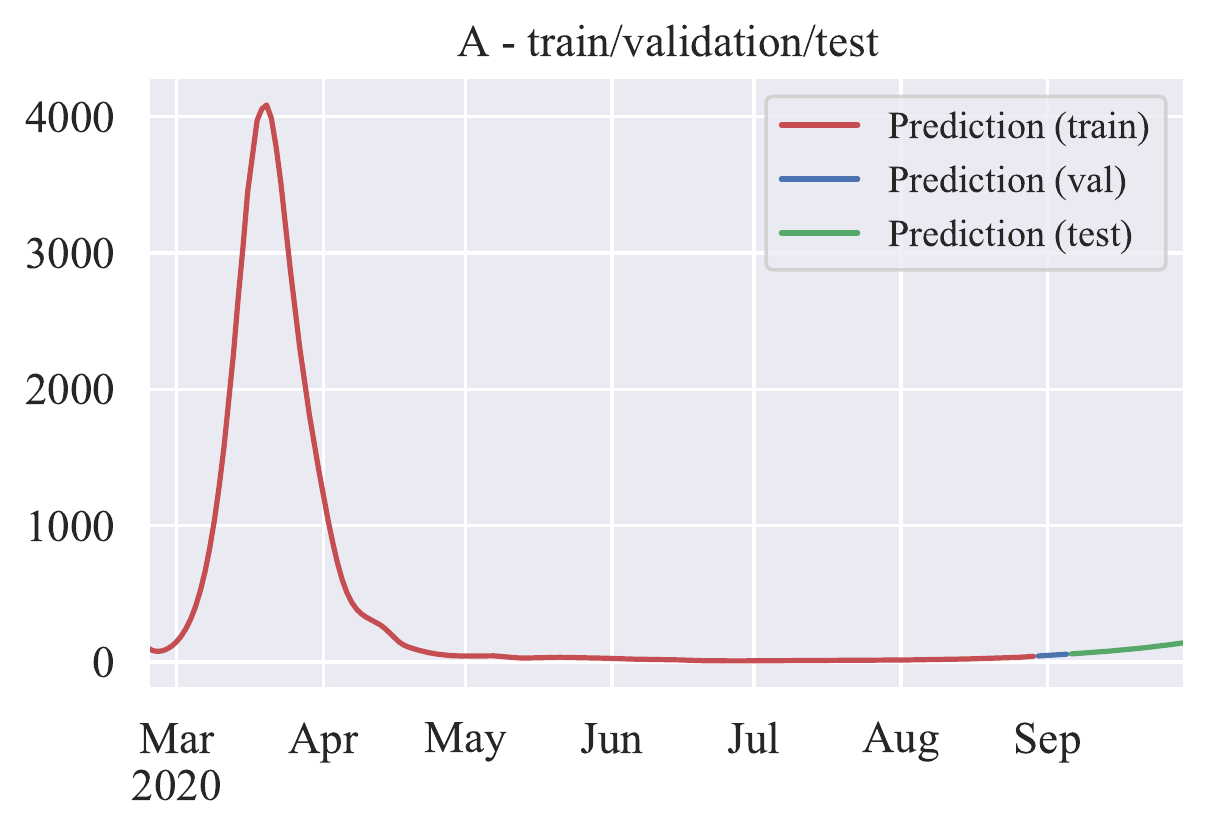|\hfil\|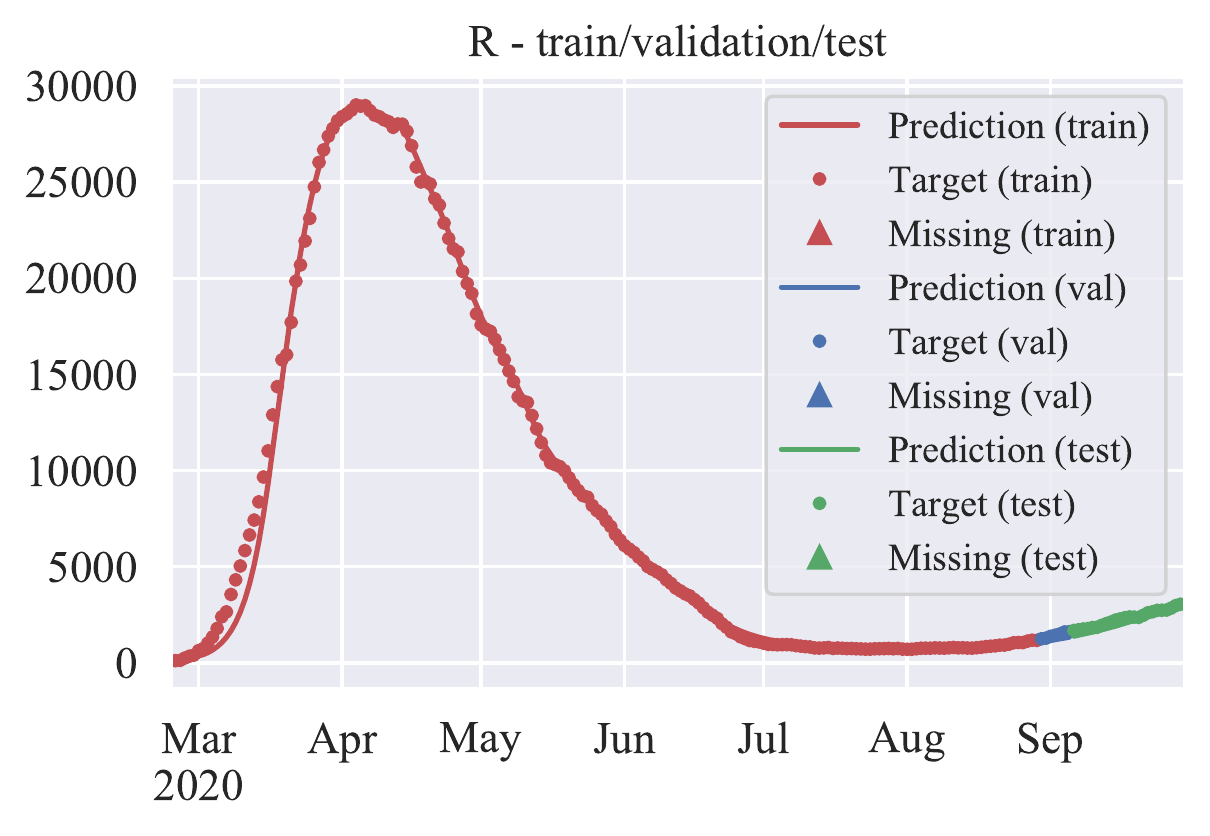|\hfil\|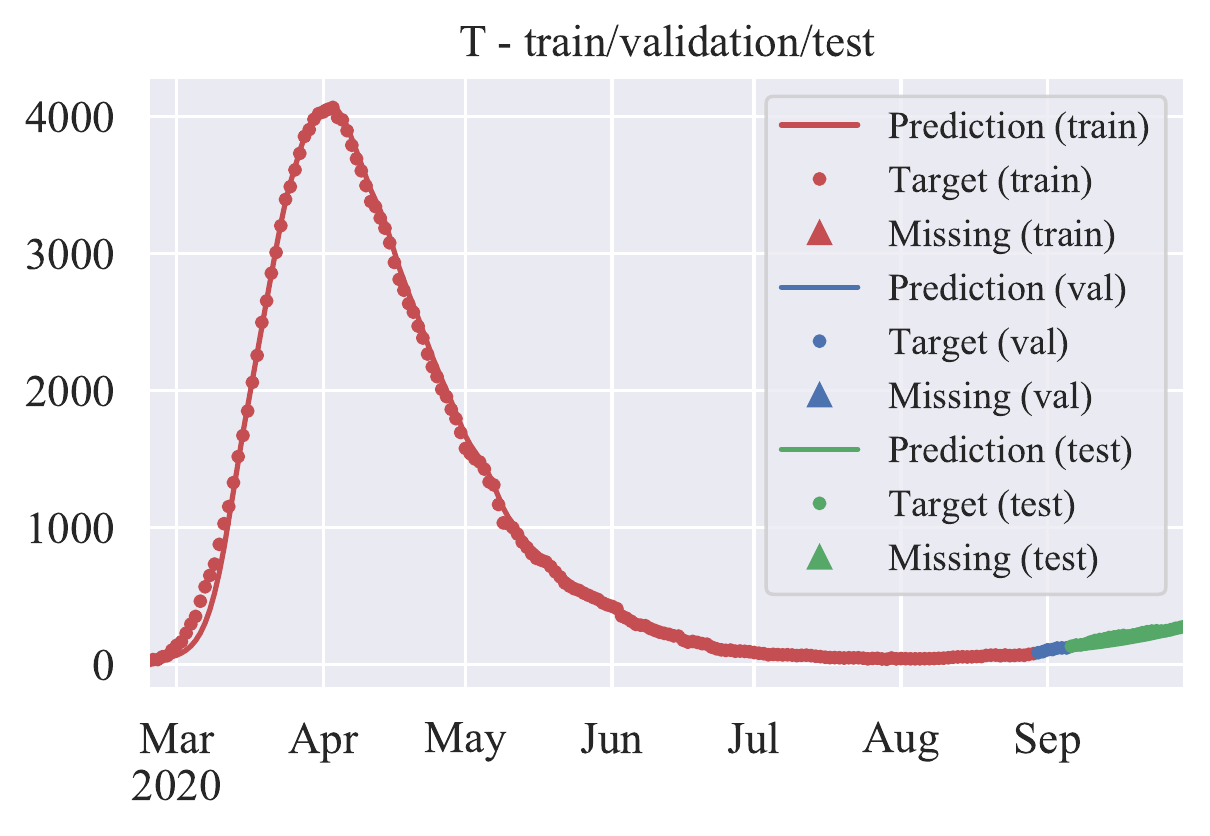|}
\hbox to\hsize{\|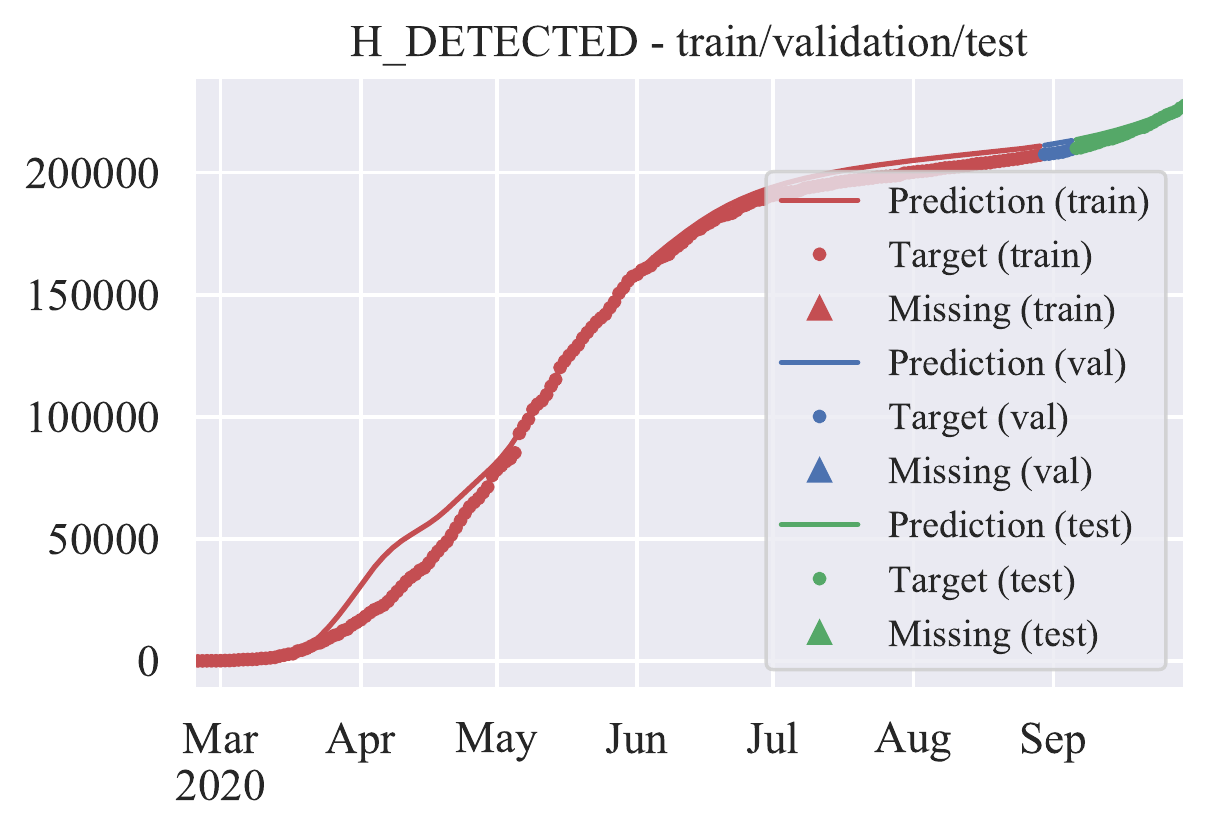|\hfil\|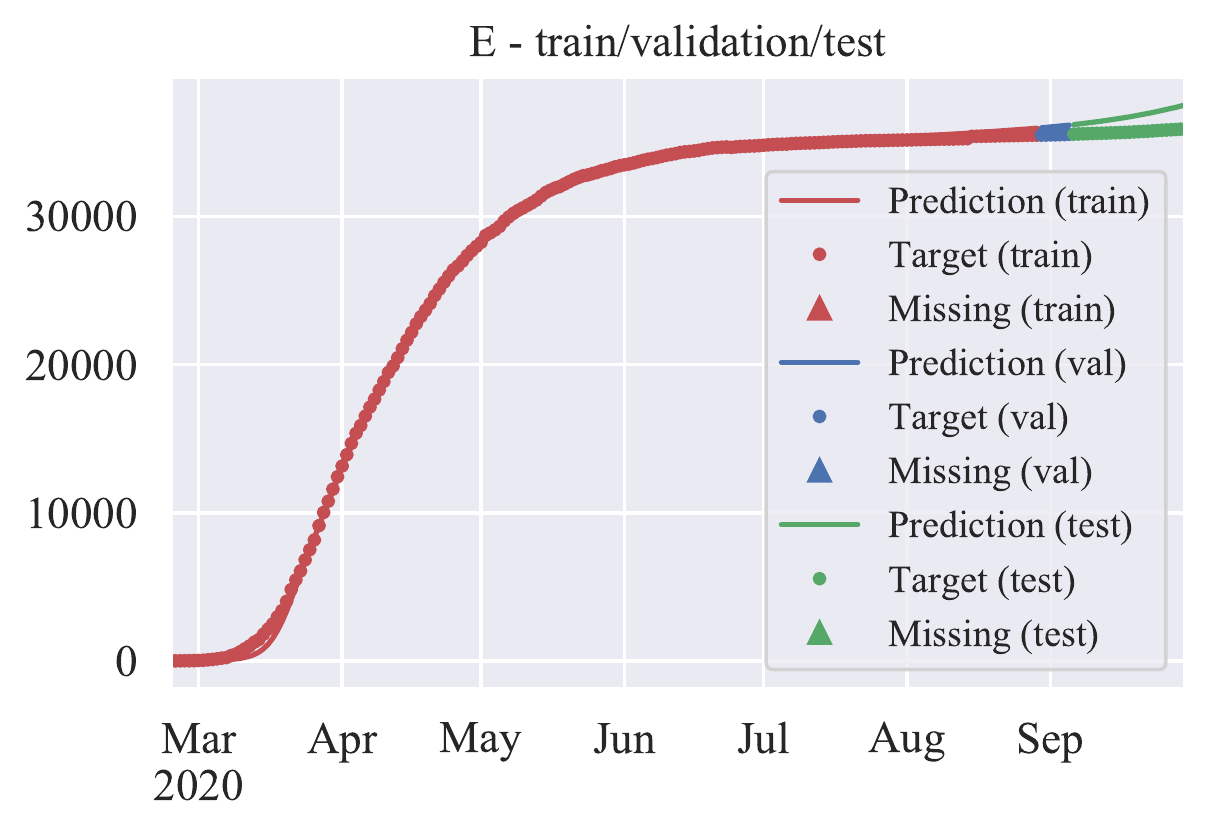|\hfil\|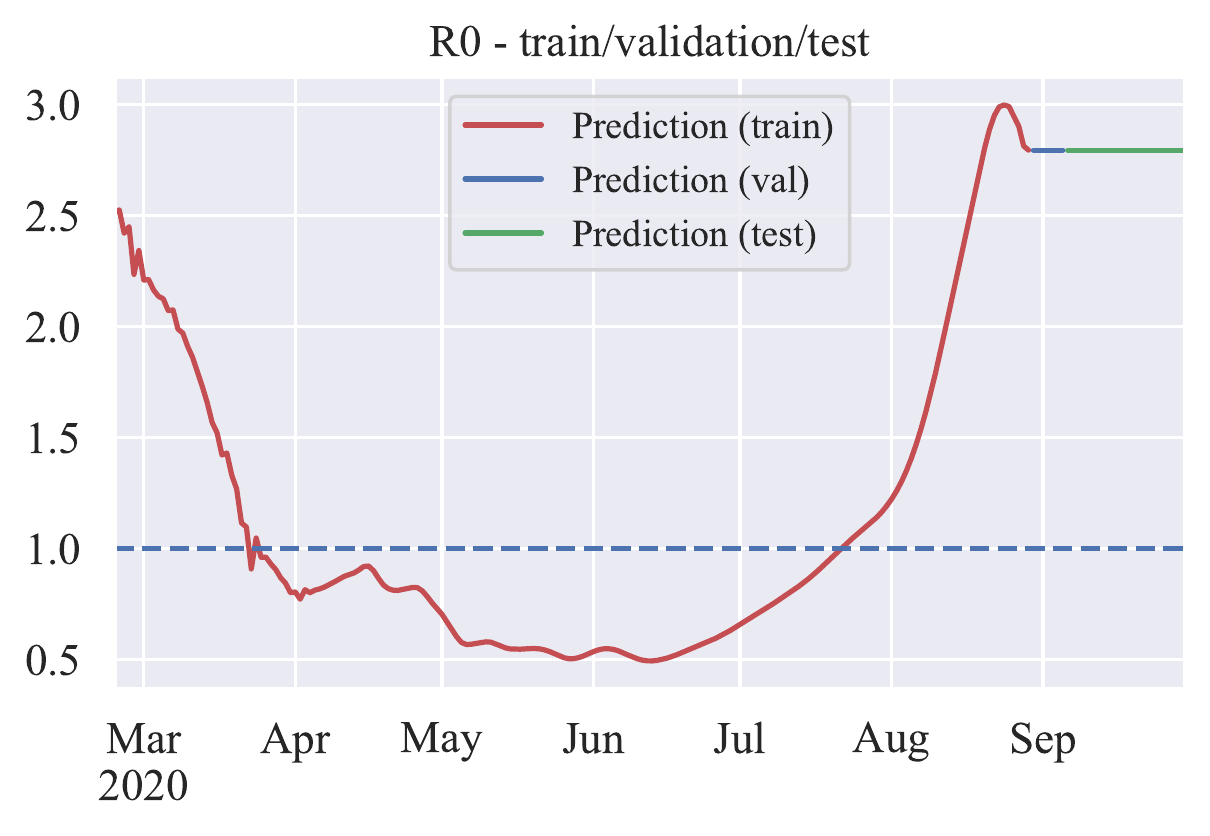|}
\caption{Epidemic evolution of COVID-19 in Italy.}
\label{fig:sidarthe_fit_it}
\end{figure*}

\def\|#1|{\includegraphics[width=0.33\textwidth]{#1}}
\begin{figure*}[!ht]
\hbox to\hsize{\|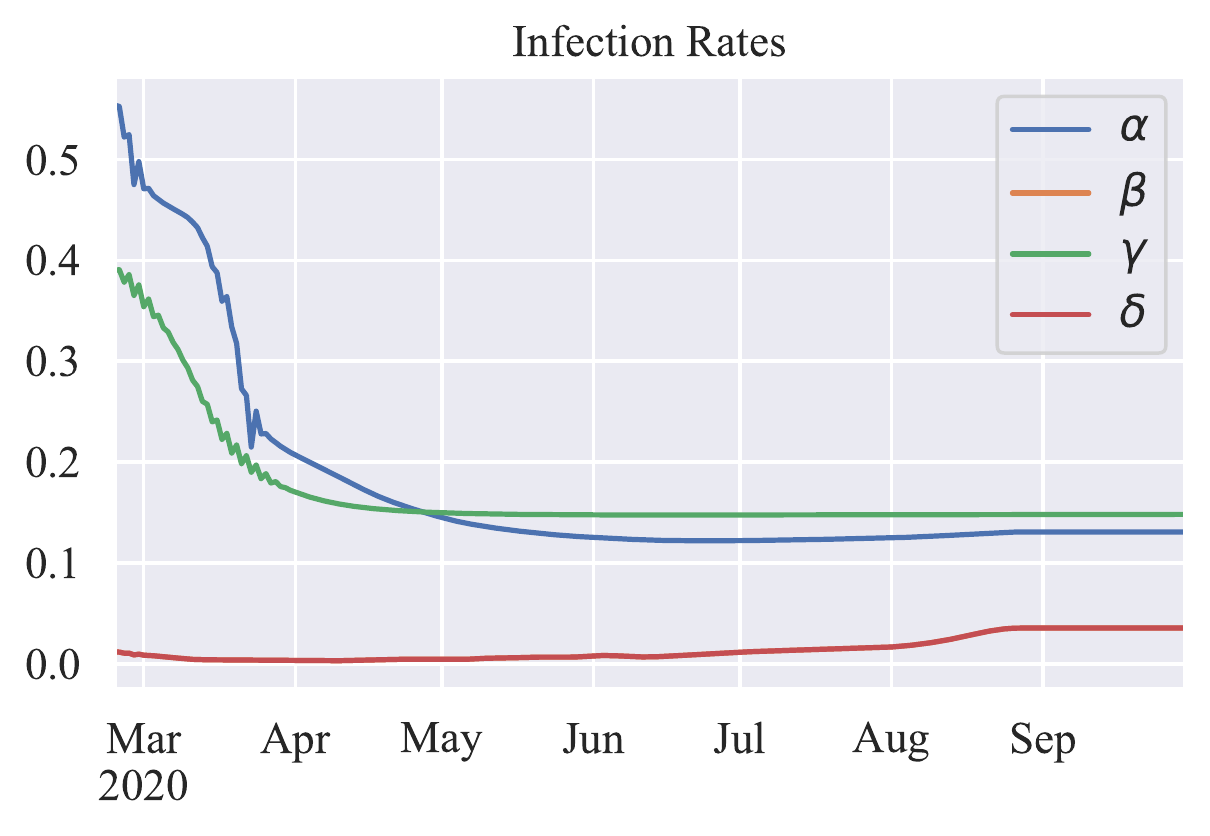|\hfil\|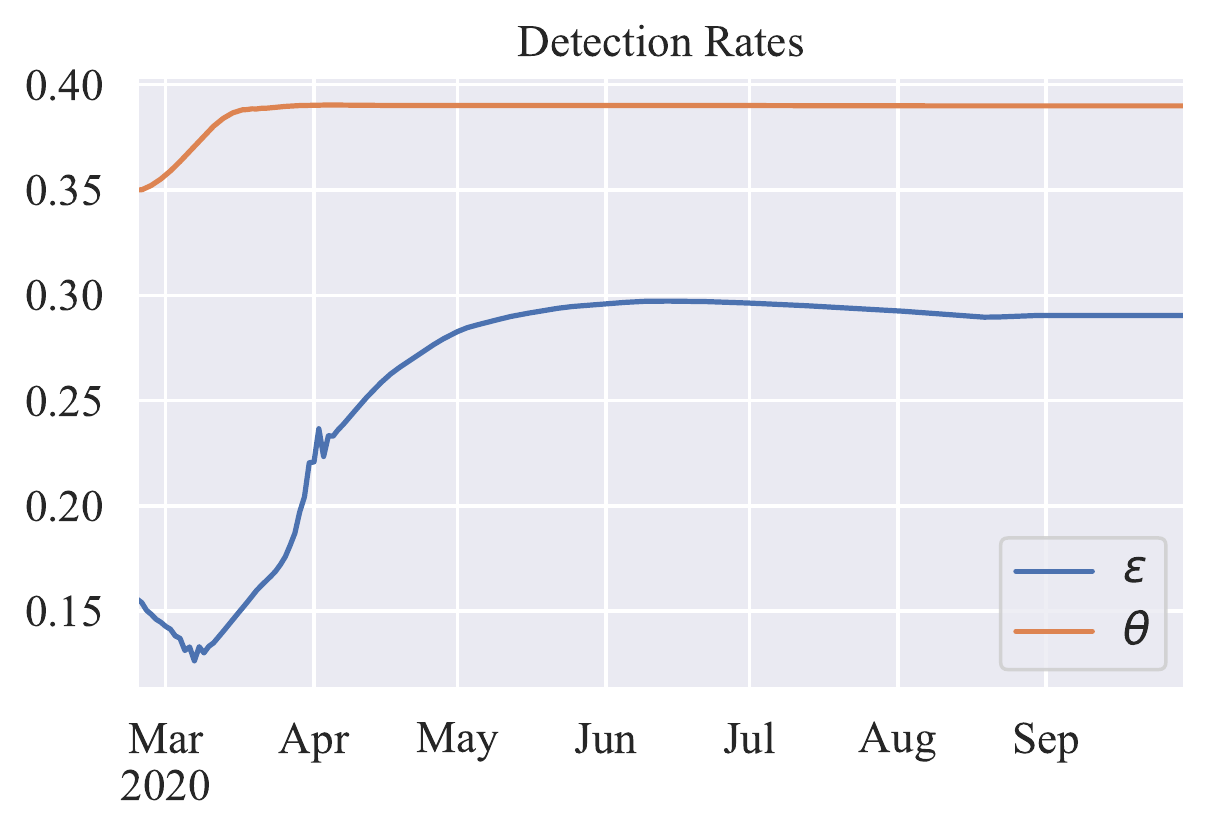|\hfil\|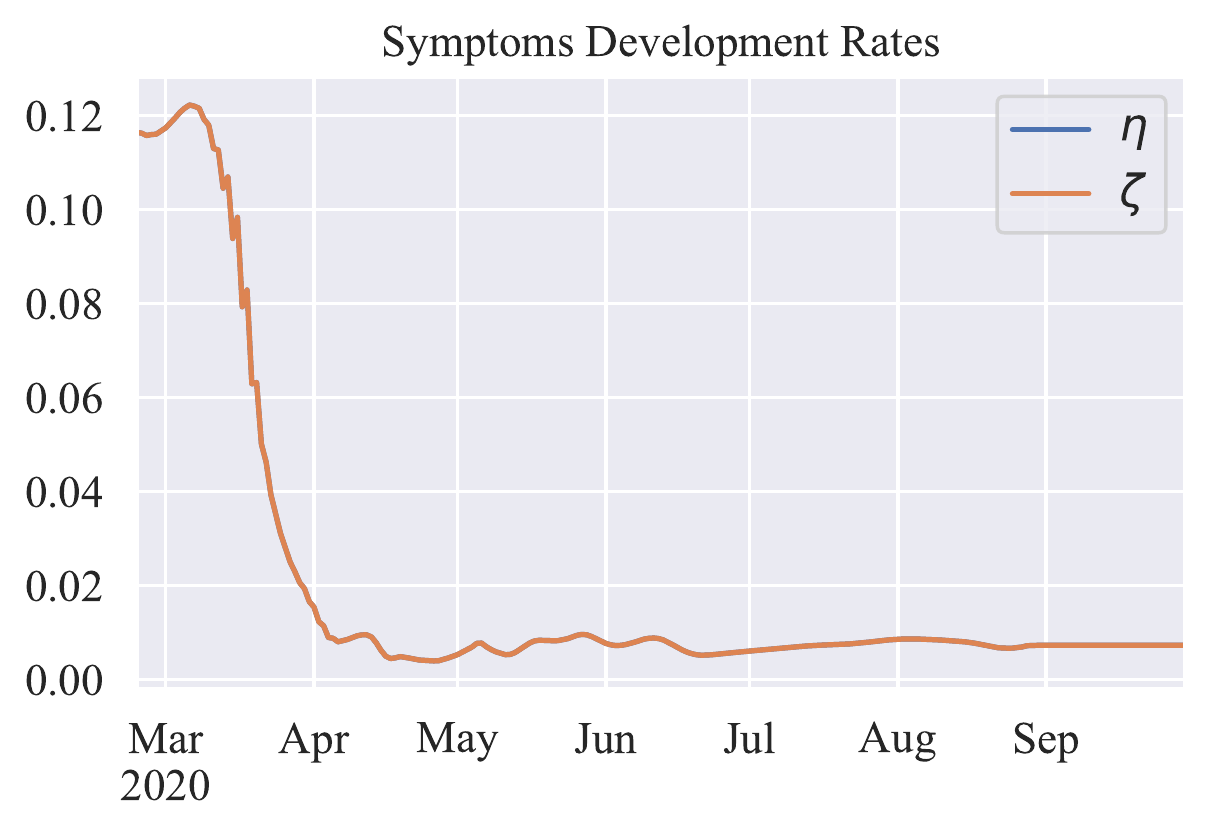|}
\hbox to\hsize{\|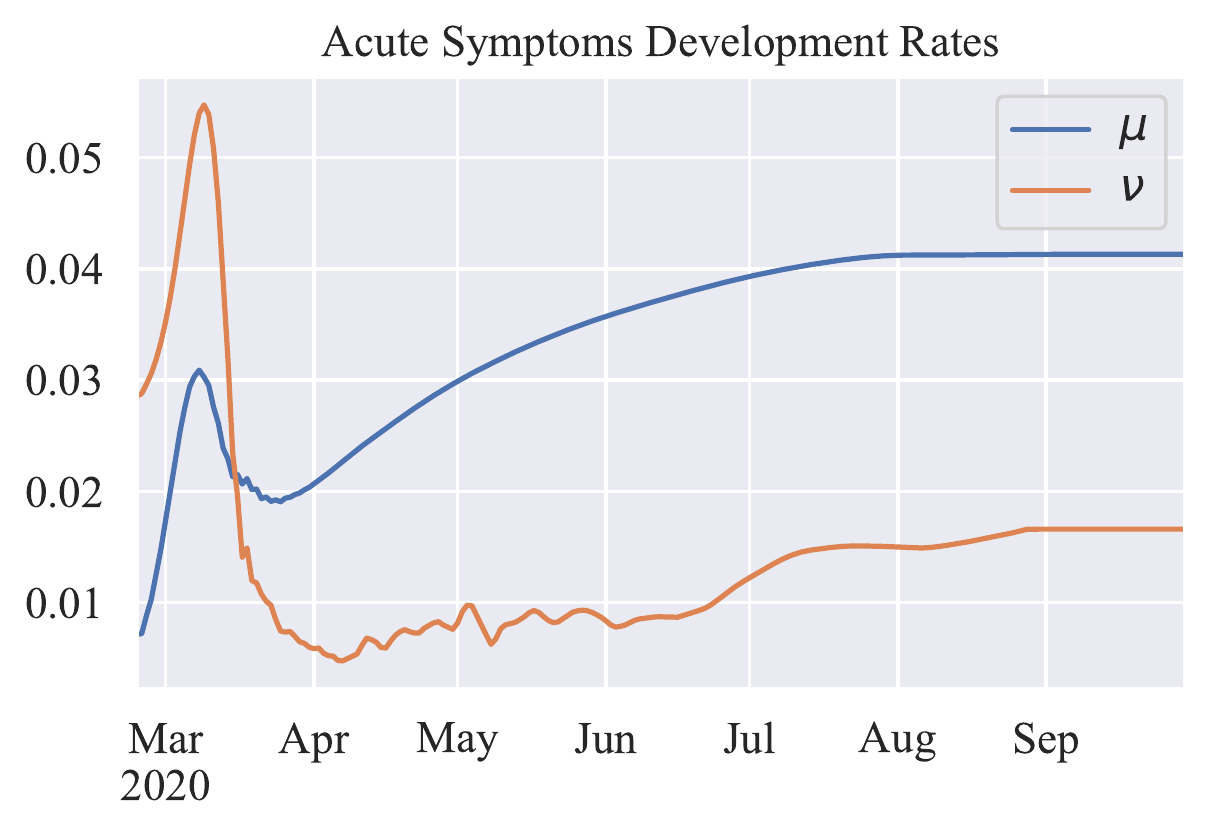|\hfil\|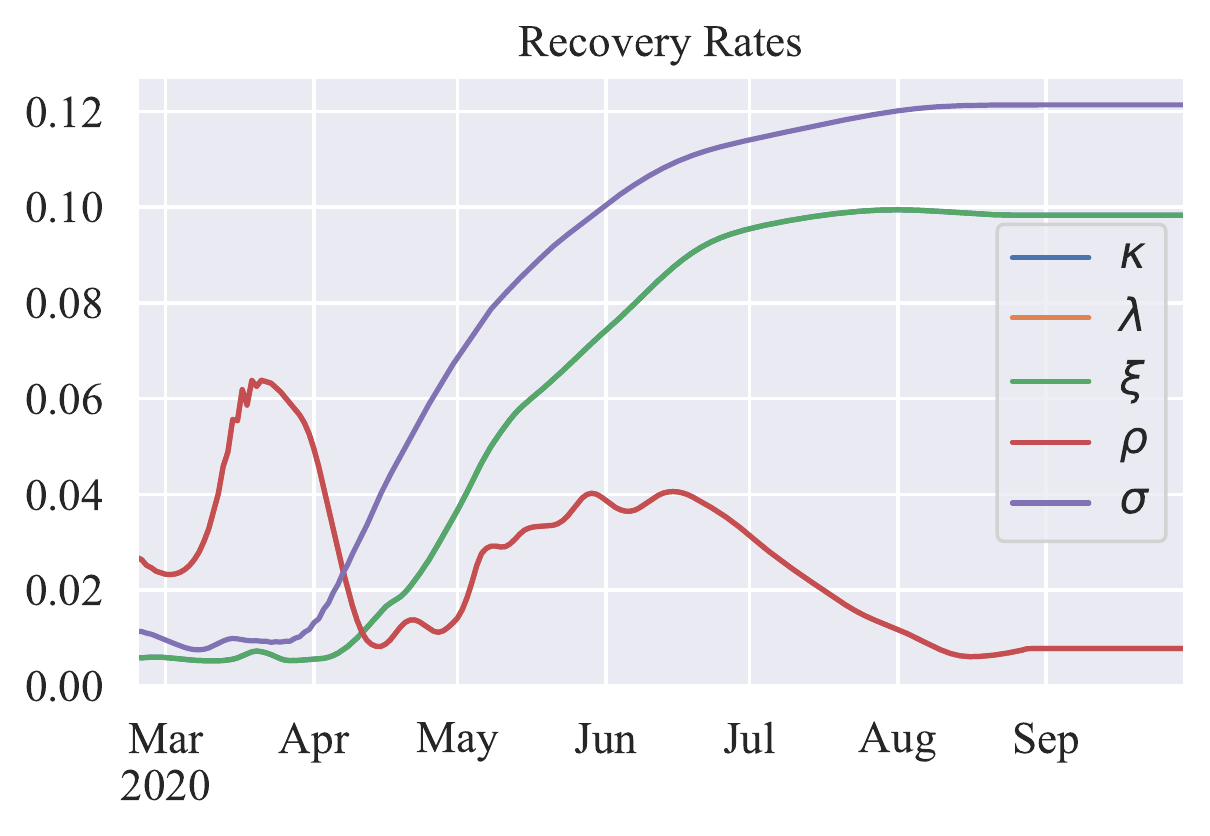|\hfil\|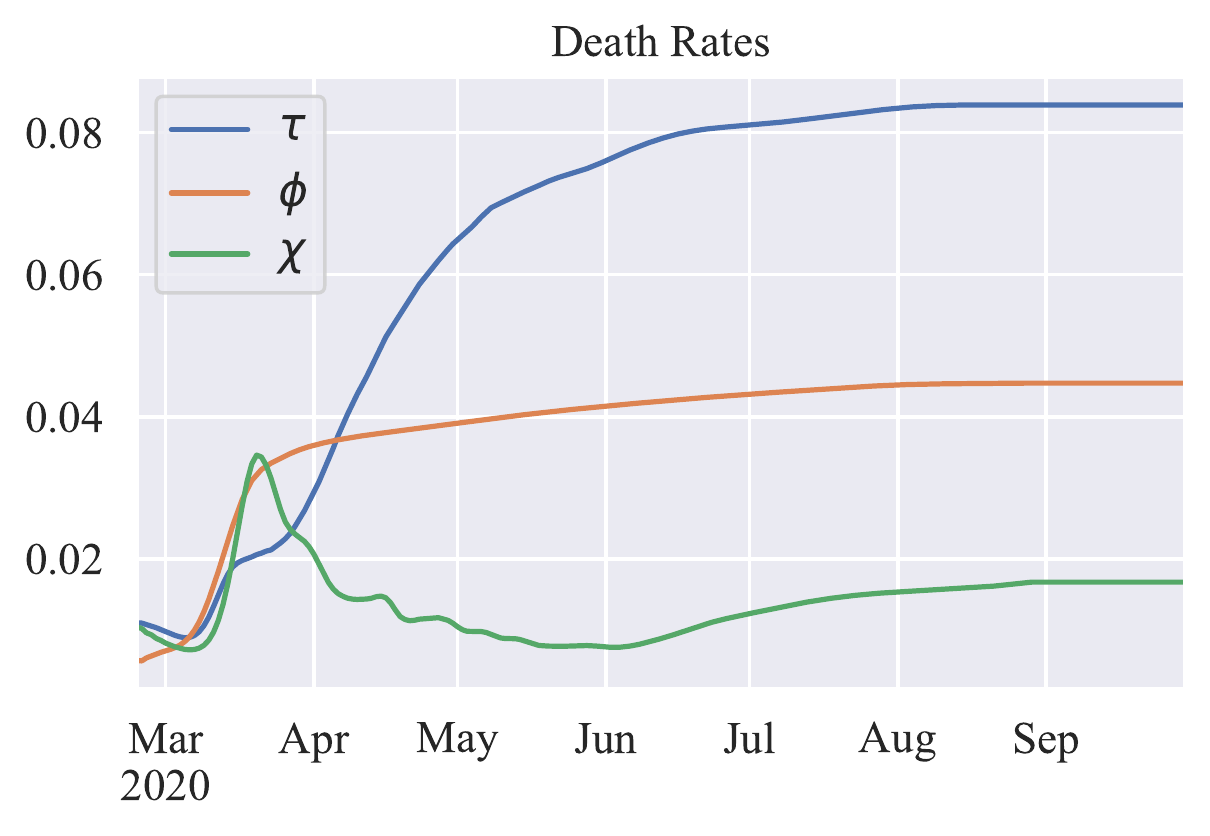|}
\caption{Time-variant parameters dynamics in Italy.}
\label{fig:sidarthe_params_it}
\end{figure*}

\def\|#1|{\includegraphics[width=0.33\textwidth]{#1}}
\begin{figure*}[!hb]
\hbox to\hsize{\|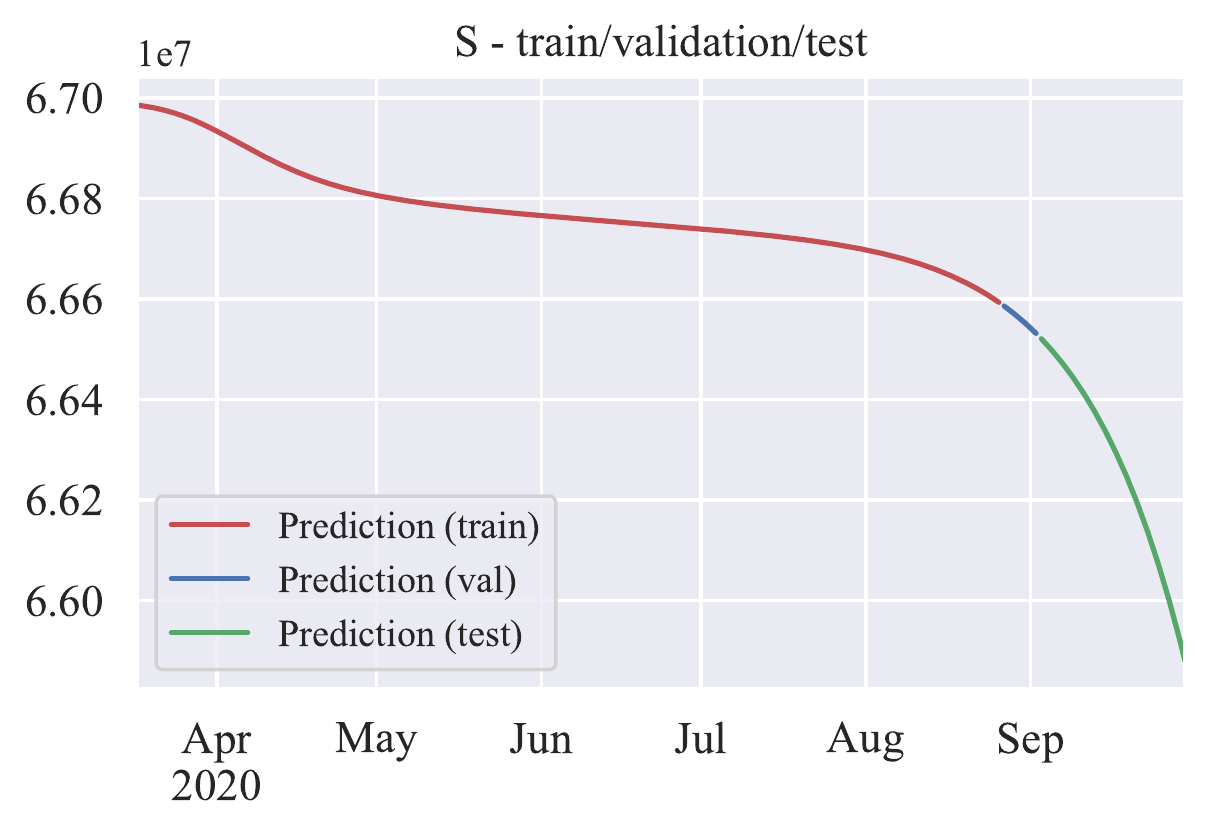|\hfil\|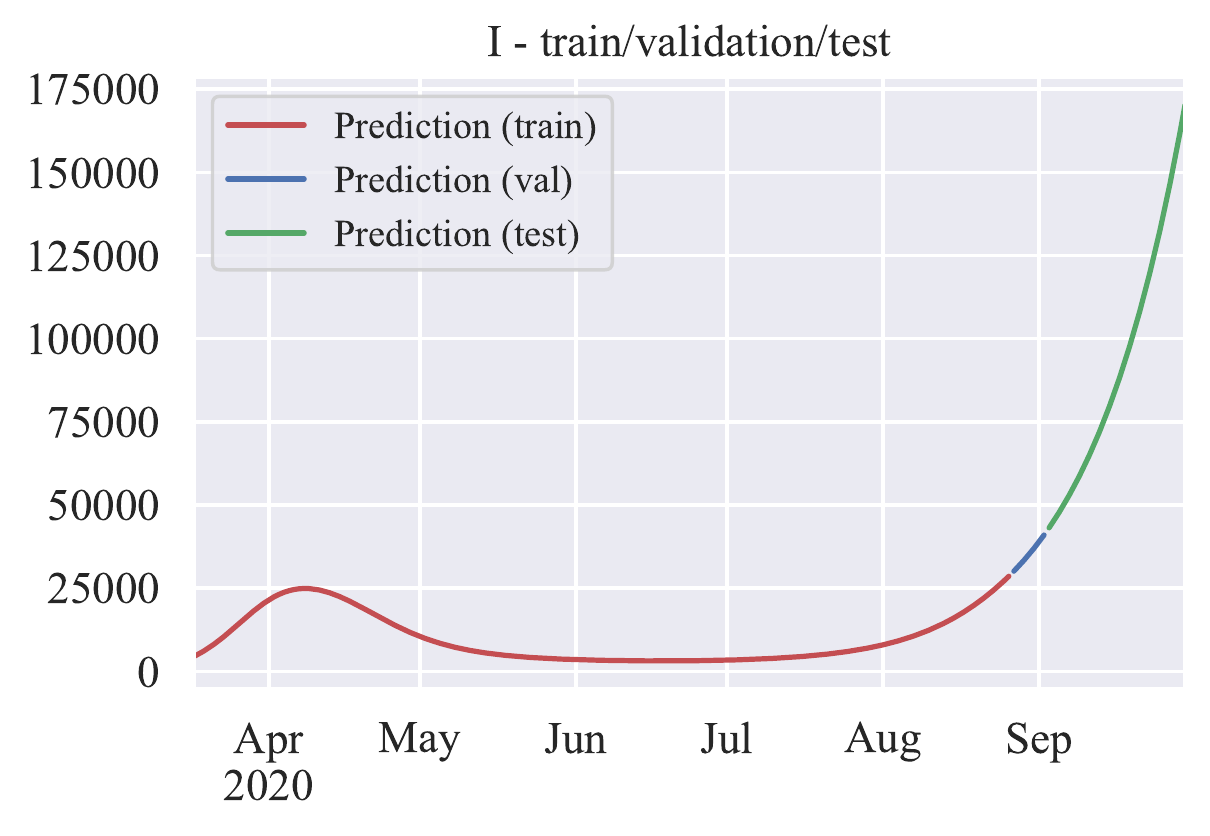|\hfil\|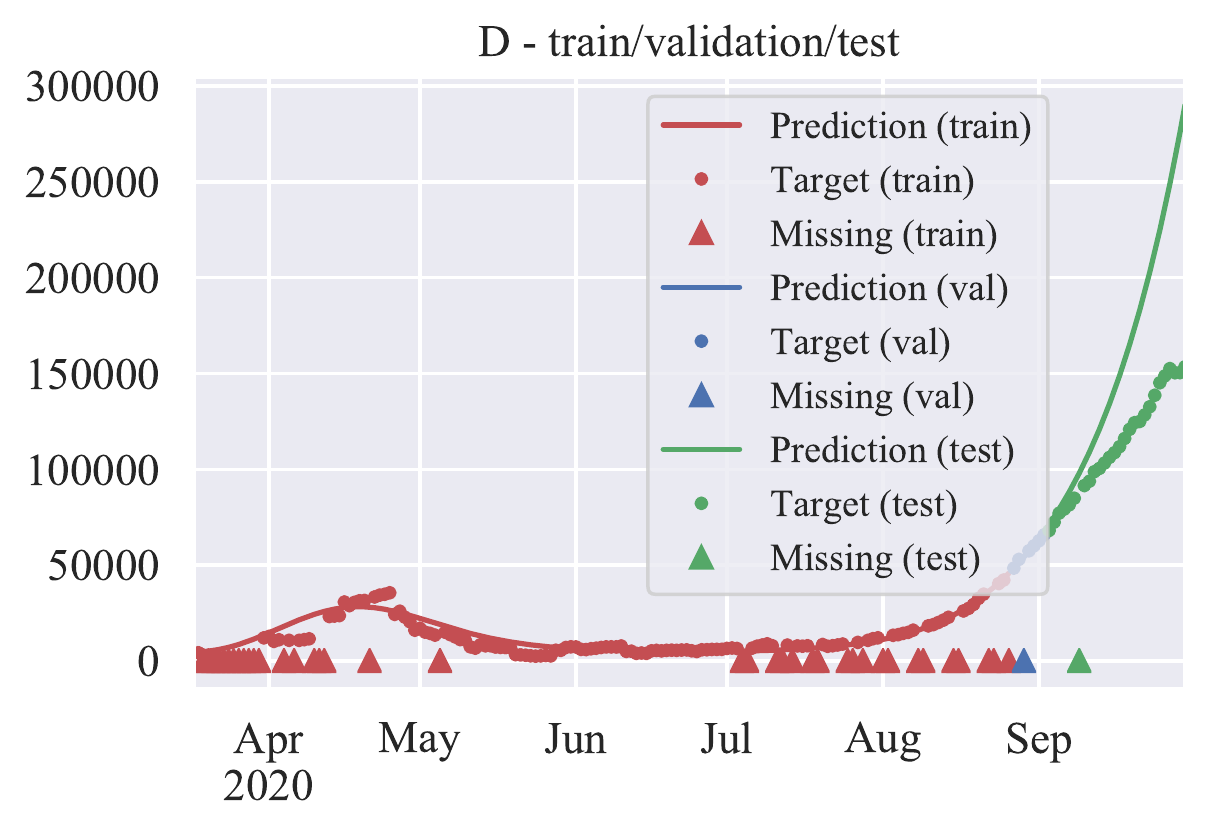|}
\hbox to\hsize{\|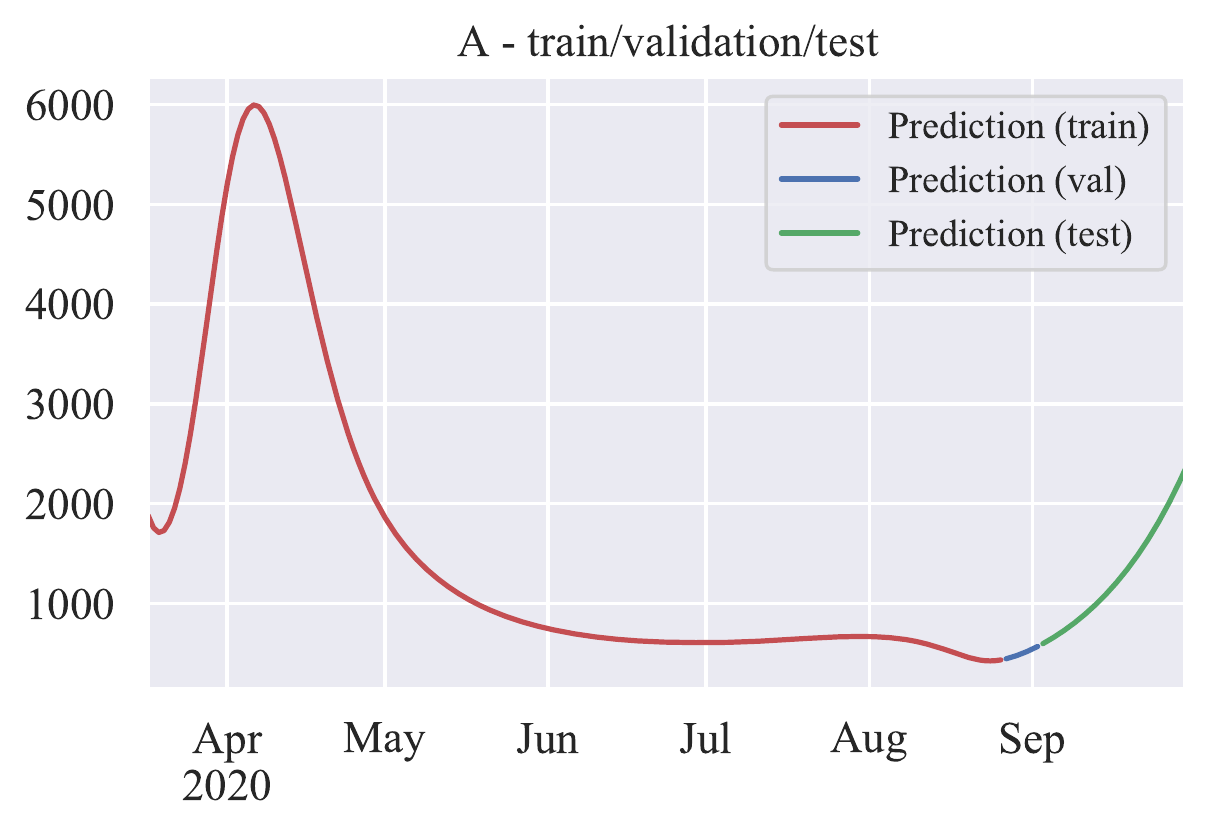|\hfil\|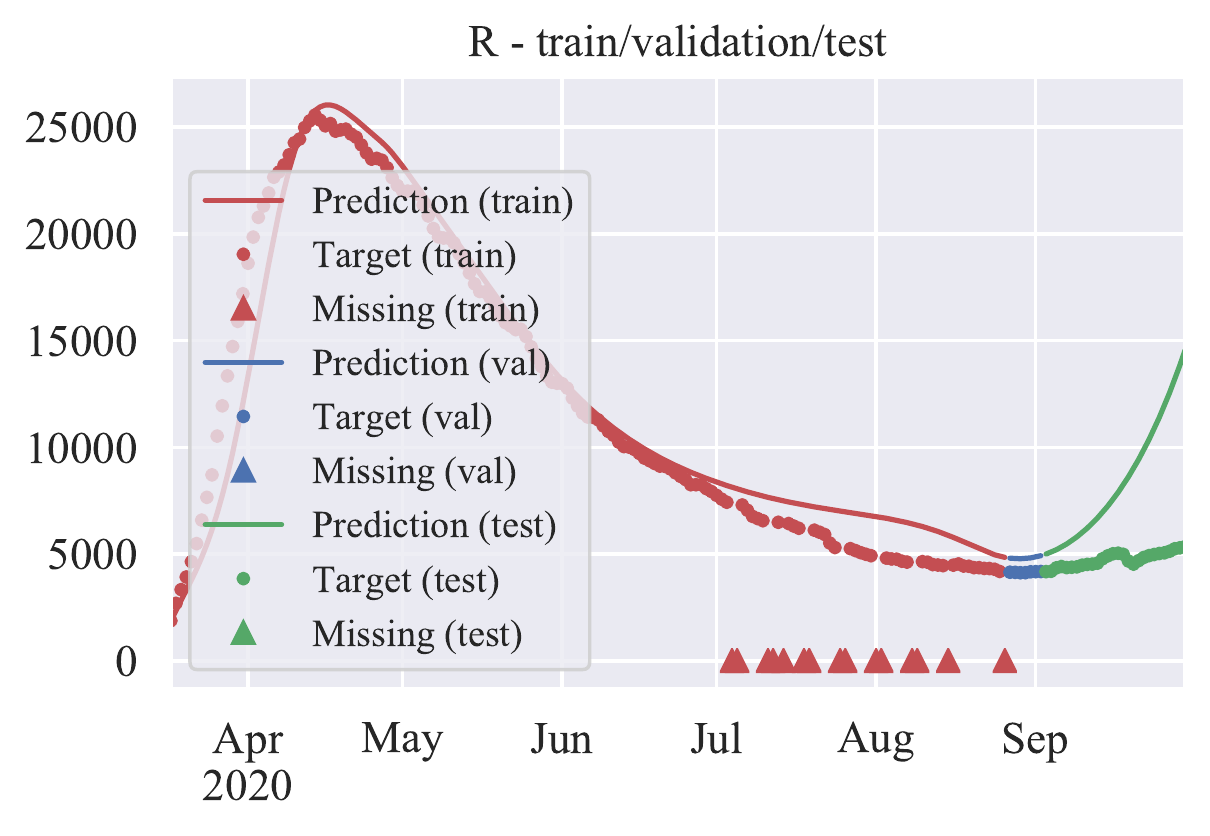|\hfil\|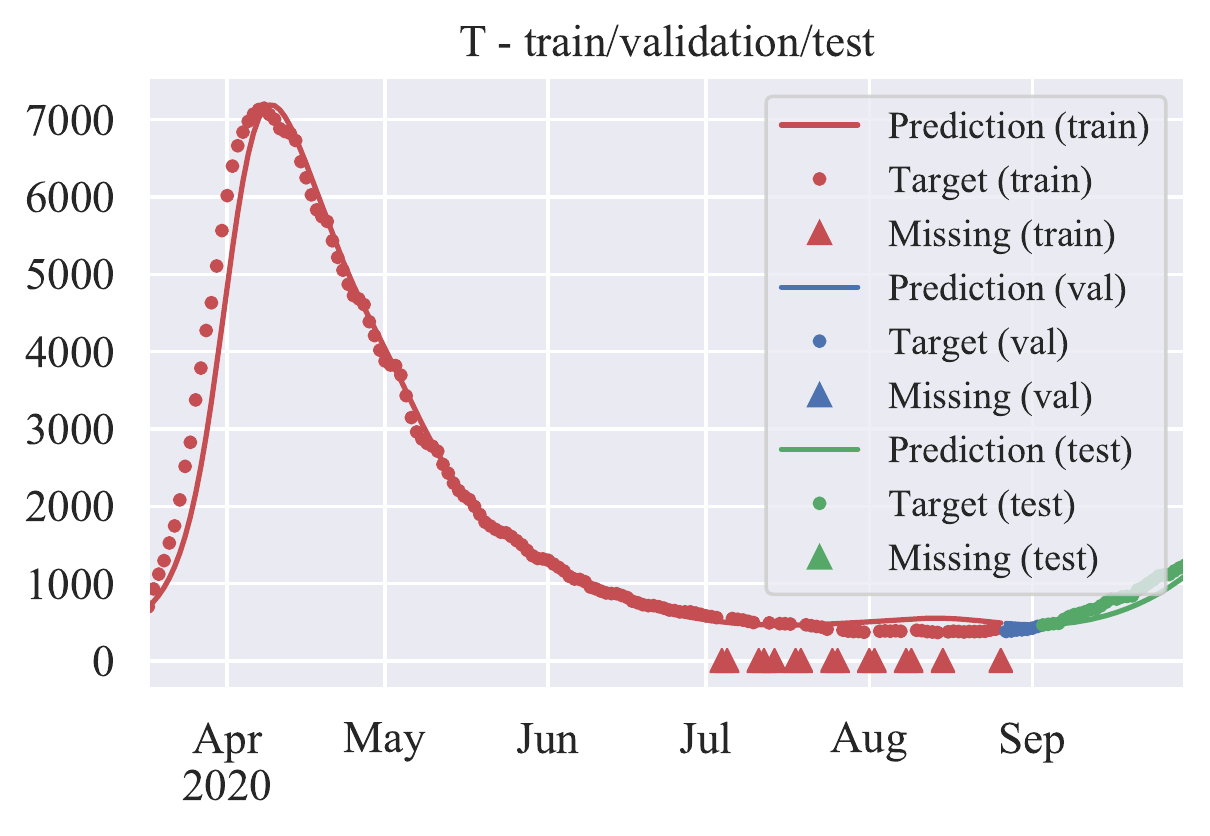|}
\hbox to\hsize{\|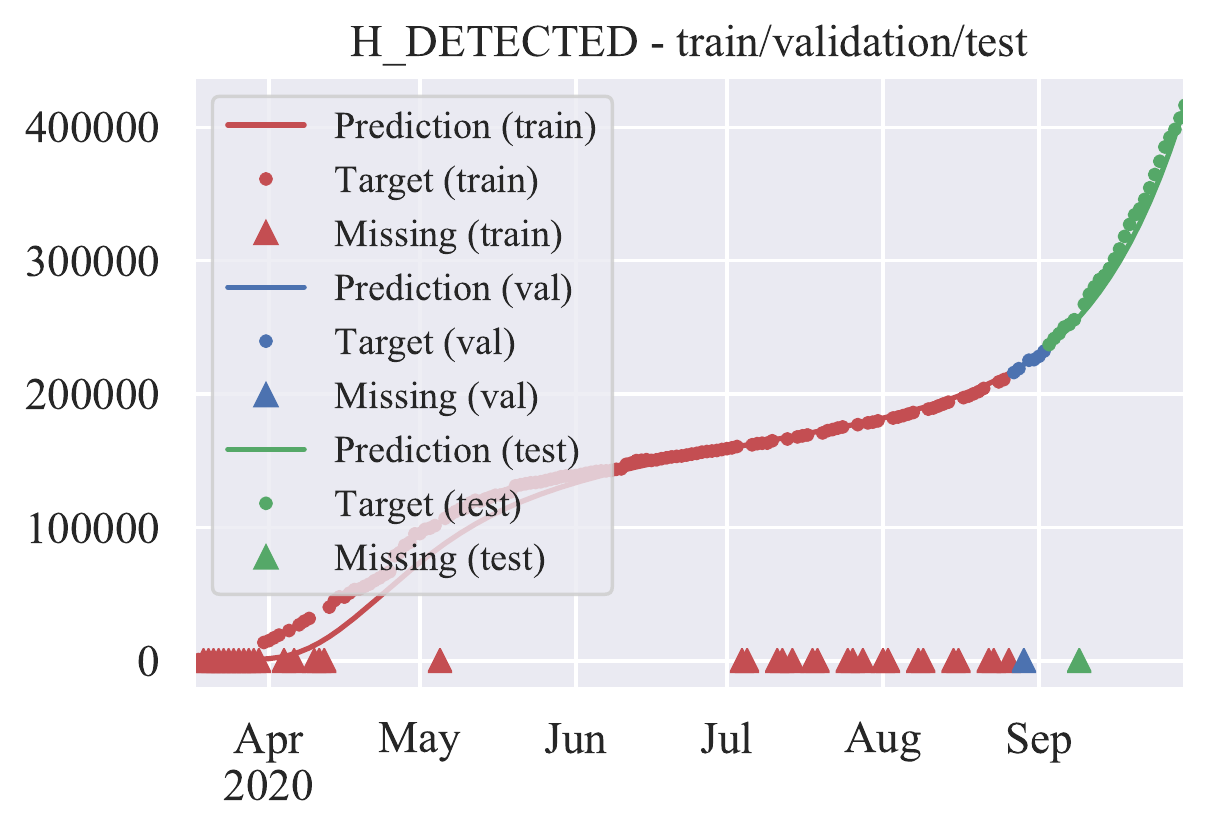|\hfil\|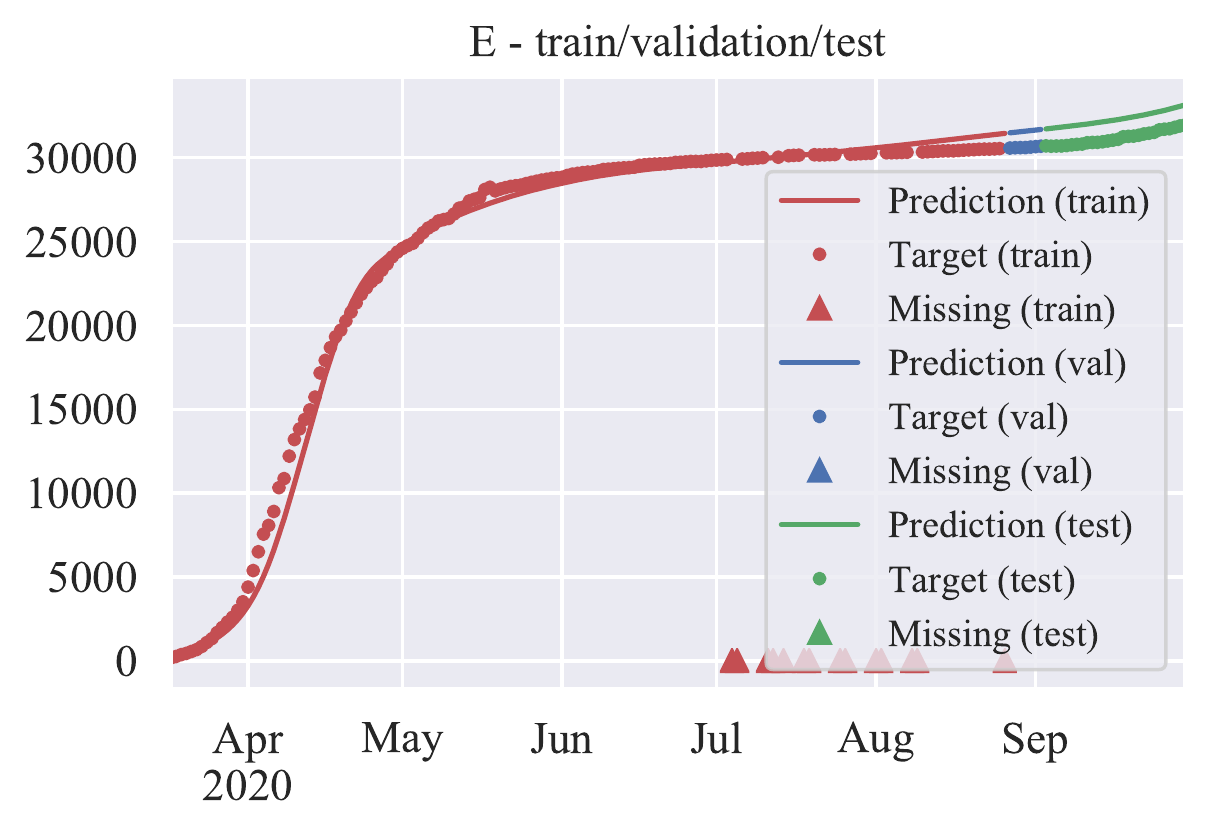|\hfil\|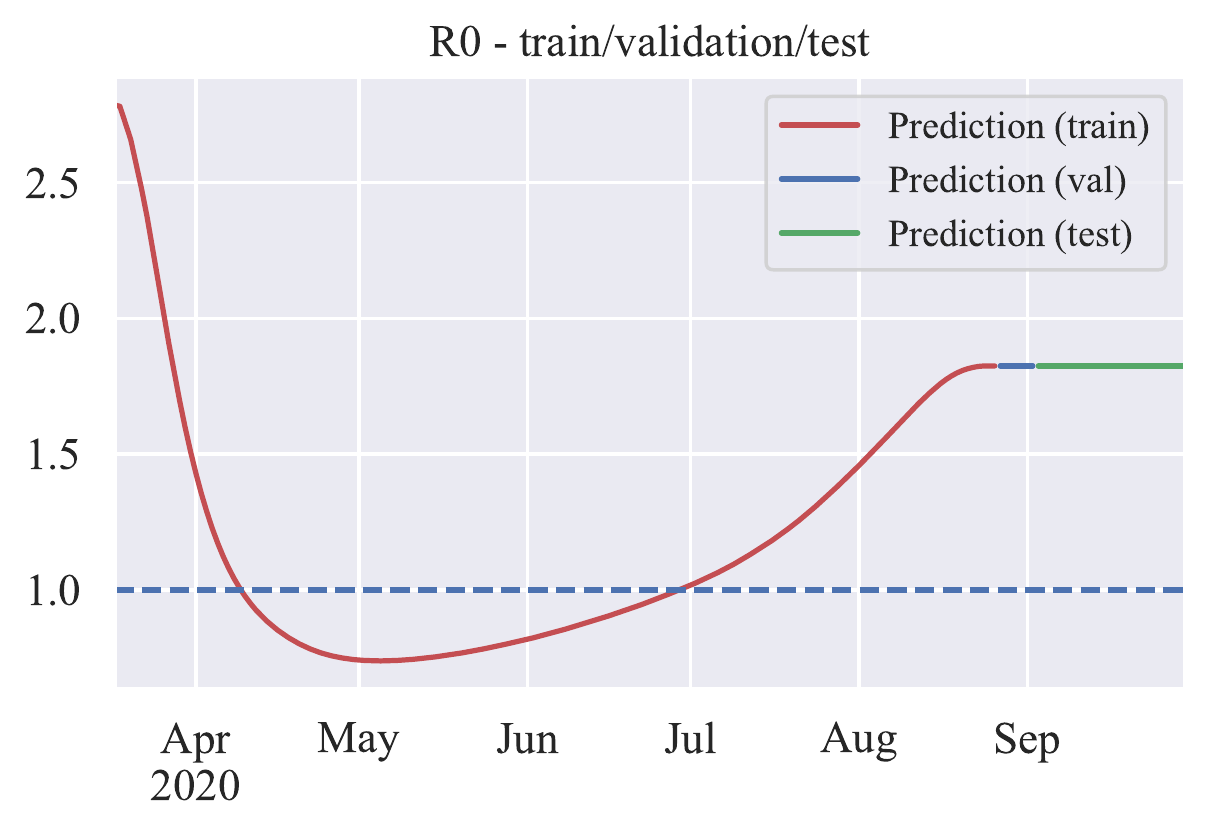|}
\caption{Epidemic evolution of COVID-19 in France.}
\label{fig:sidarthe_fit_fr}
\end{figure*}

\def\|#1|{\includegraphics[width=0.33\textwidth]{#1}}
\begin{figure*}[!ht]
\hbox to\hsize{\|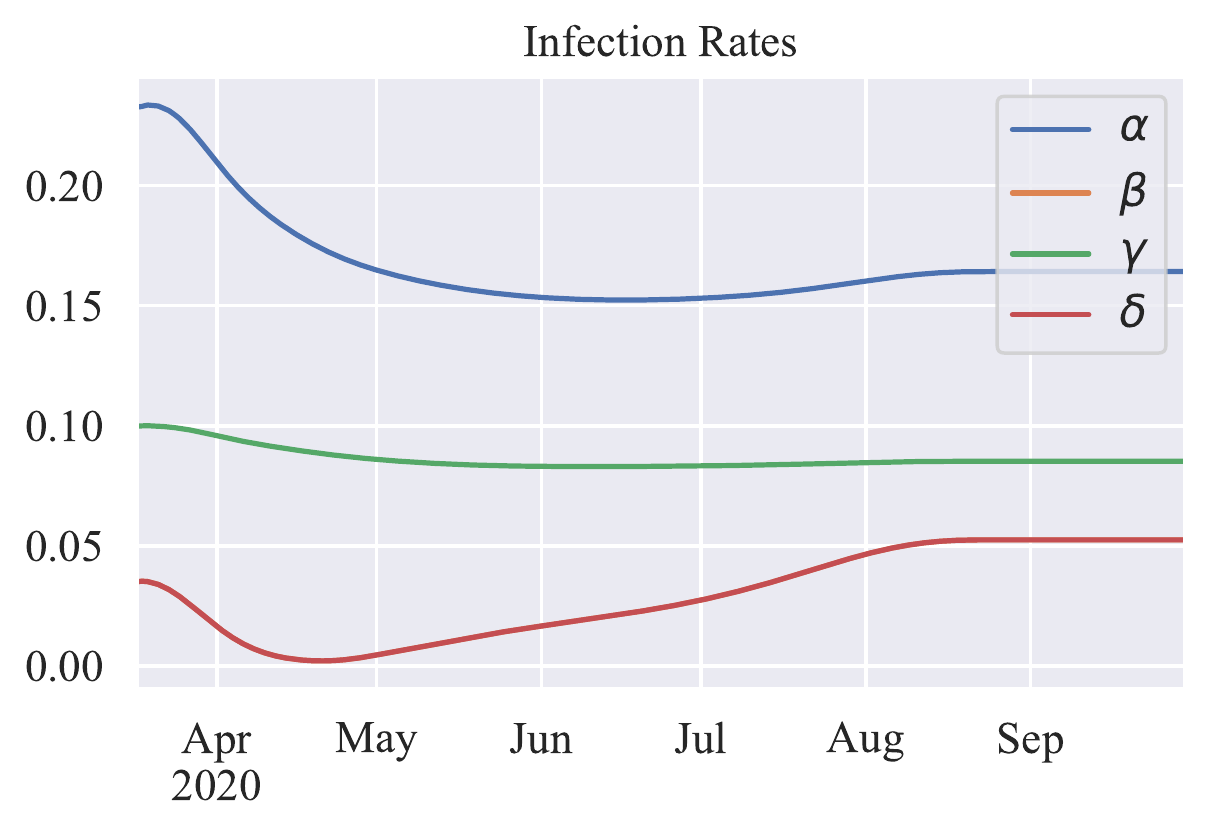|\hfil\|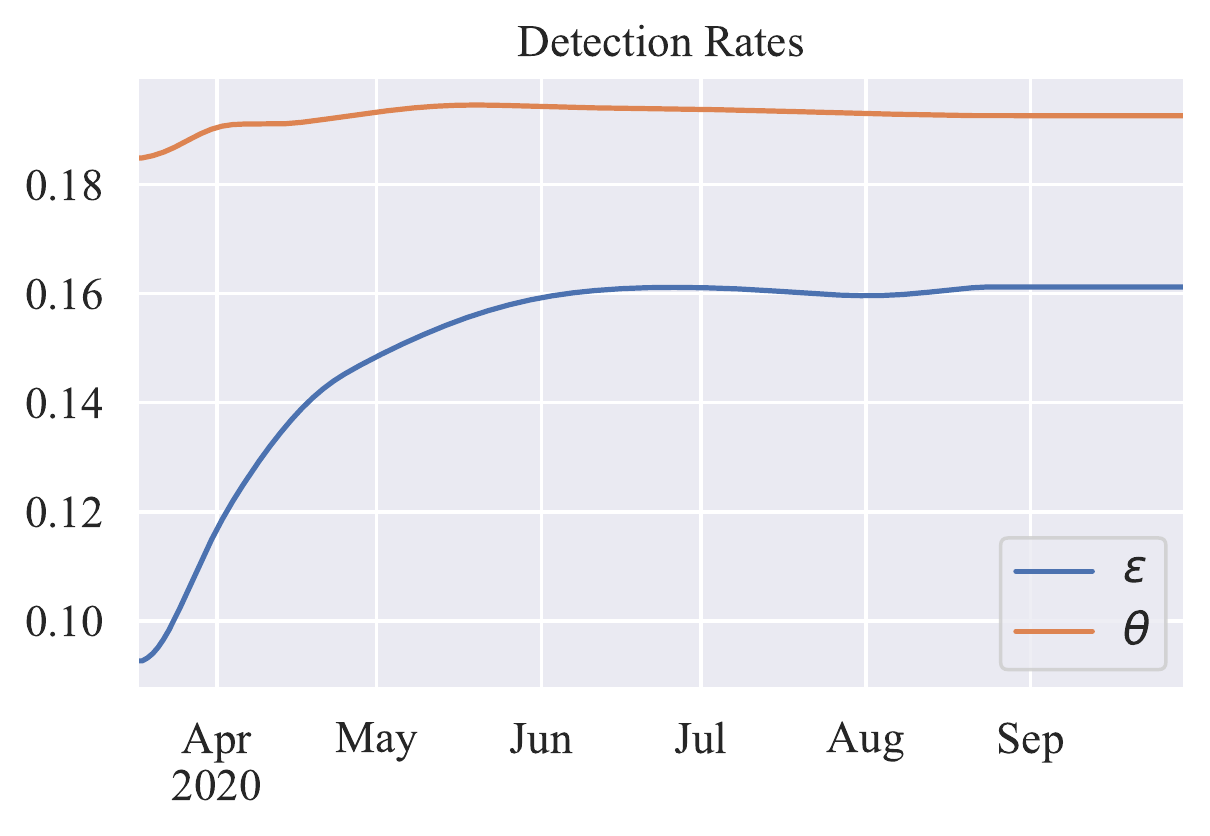|\hfil\|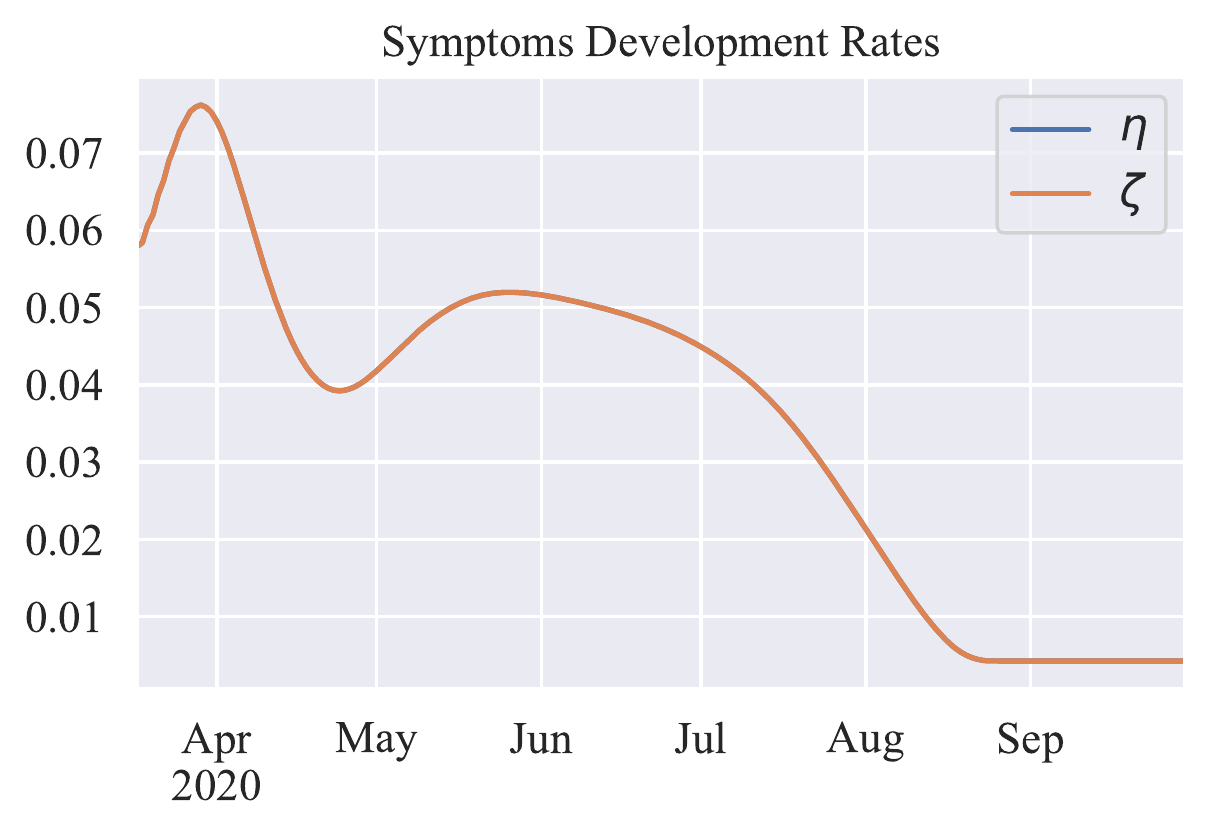|}
\hbox to\hsize{\|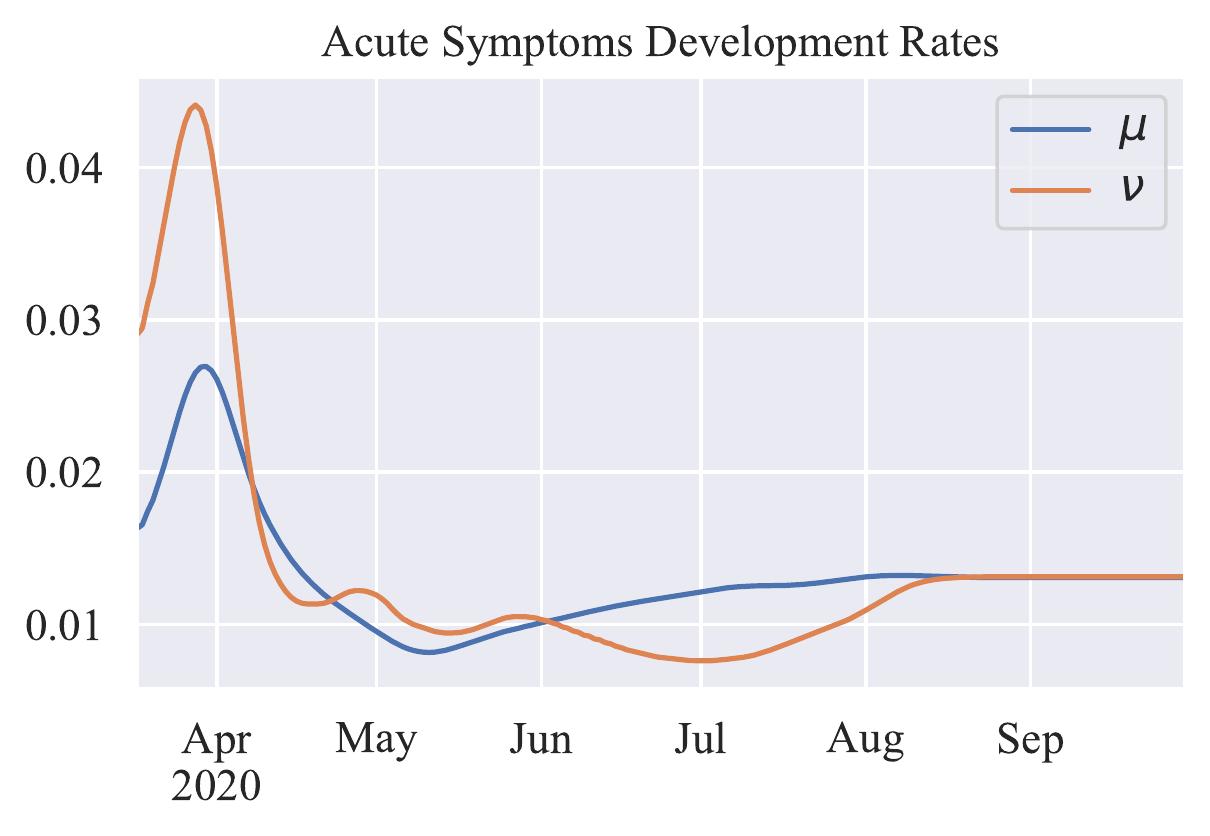|\hfil\|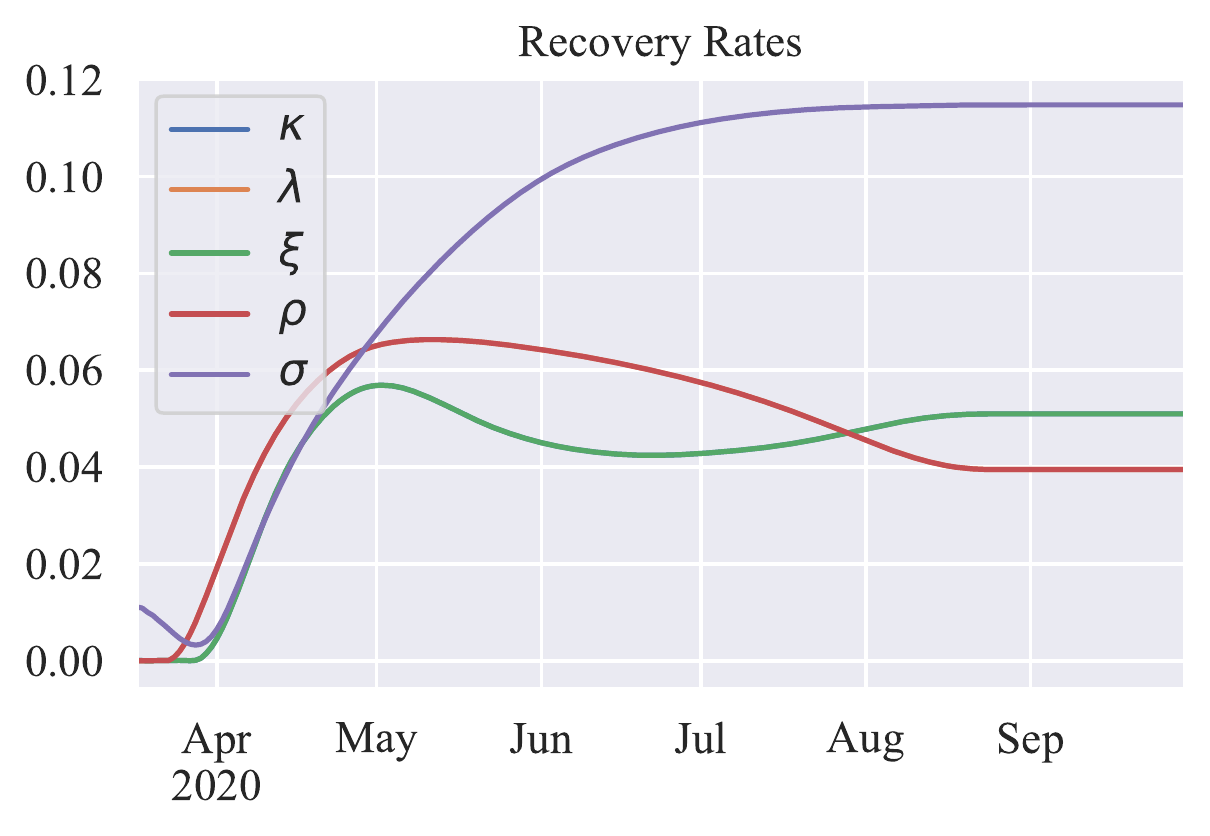|\hfil\|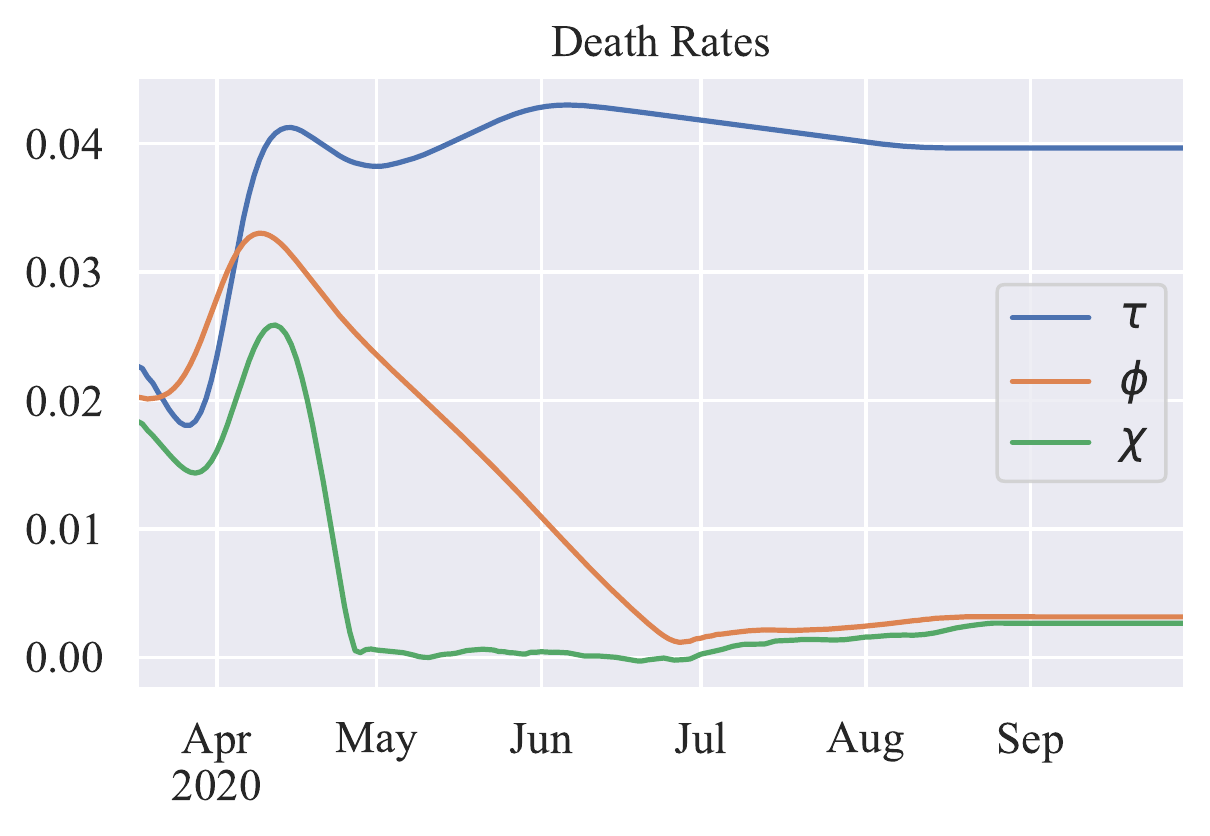|}
\caption{Time-variant parameters dynamics in France.}
\label{fig:sidarthe_params_fr}
\end{figure*}


In addition to restricting the space of the solutions to the problem in Eq.\eqref{eq:min-problem} by means of the first order constraint (previously discussed),
we furthermore decided to reduce the number of learnable parameters. 
The learning of some pairs of parameters was tied together. In particular we tied $\beta$ and $\delta$, $\xi$ and $\kappa$, $\lambda$ and $\rho$, $\eta$ and $\zeta$. 
As initial conditions of the dynamical system we used the following values: $\i^0=\d^0=\hat{\d}(0)$, $\a^0=\r^0=\hat{\a}(0)$, $\t^0=\hat{\t}(0)$, $\h_d^0 =\hat{\h}(0)$, $\e^0 = \hat{\e}(0)$, $\s^0 = N - (\i^0+\d^0+\a^0+\r^0+\t^0+\h^0+\e^0)$, where $N$ is the size of the population considered.
We found that starting from a good initialization of the parameters $u(t)$ facilitates the learning and leads to better results.
In the Italian case, we initialized all the parameters with the values provided in \cite{Giordano2020},
whereas for the French data set we initialised $u(t)$ as a constant (not time dependent) such that $R_0=1.95$.

\begin{table}
\centering
\caption{Model forecast on Italian and French Test data. Mean Absolute Percentage Error (MAPE), and the fraction of days $\th$ where the predictions are within an error threshold of $30\%$.}
    \begin{tabular}{lccccc}\toprule
         & \multicolumn{2}{c}{Italy} & & \multicolumn{2}{c}{France}\\
         \noalign{\smallskip}
         & MAPE  & $\th$ & & MAPE  & $\th$\\\midrule
        $\d$ & 16\% & 20/25 & & 41\% & 10/25 \\
        $\r$ &  8\% & 25/25 & & 84\% & 3/25 \\
        $\t$ & 19\% & 25/25 & & 16\% & 25/25 \\
        $\h$ &  4\% & 25/25 & &  2\% & 24/24 \\
        $\e$ &  6\% & 25/25 & &  5\% & 25/25 \\
        \bottomrule
    \end{tabular}
\label{tab:it_test_forecast}
\end{table}



We performed model selection based on the best solution in the validation period.
The best models were obtained through grid search in the space of the hyper-parameters. 
In particular the positive constants $e_\t$, $e_\r$, $e_\d$, $e_\h$, $e_\e$ that weigh the 
terms $\overline\t$, $\overline\r$, $\overline\d$, $\overline\h_d$, $\overline\e$ in the functional risk $F$, 
the coefficient $m$ that acts on the derivative term $|\dot\cvar|^2/2$, the factor $e_p$ that
enforces the positivity of the solutions, the 
parameters $a$ and $b$ that define the $\omega$ function in Eq.~\eqref{eq:update_rule_params} span the hyper-parameters space of the learning method.
Based on the findings of the ablation study in Section~\ref{subsec:ablation}, we can narrow the search in the hyper-parameters space by setting $a=0$, $b\in [0.05,0.125]$ and $m\in[10^5,10^{11}]$. The 
learning rate $\pi_0$, was set to $10^{-5}$.

In order to provide a better understanding of the predictive capabilities of our model, we also report in Table~\ref{tab:it_test_forecast}
the Mean Absolute Percentage Error (MAPE) and the fraction of \textit{test} days where the model predictions are beyond a certain tolerance error threshold, that we call $\th$.  We conclude this section with some remarks specific to each data set.

\paragraph{Italy.} The epidemic spreading in Italy is showed in Fig.~\ref{fig:sidarthe_fit_it}. It turns out that the model predictions are quite accurate over windows of a few weeks. The MAPE is on average always under $20\%$ for each state variable, moreover it remains below the tolerance threshold of $30\%$ in the test with the exception of the last 5 days of $\d$ (see Table~\ref{tab:it_test_forecast}).
The obtained basic reproduction number $R_0$ reflects consistently the epidemic spreading and its values are coherent with the results reported in \cite{cintia2020relationship} for single Italian regions.
It is worth mentioning that in this paper $R_0$ refers 
to the system dynamics interpretation associated with SIDARTHE model,
which might somewhat depart from other estimations. 
%
All the model parameters are presented in Figure~\ref{fig:sidarthe_params_it}. Interestingly enough, recovery rates $\sigma$, $\xi$ (tied with $\kappa$), $\rho$ (tied with $\lambda$) associated to ICU patients, detected and undetected symptomatic individuals, respectively, steadily increase over time, suggesting that hospitals are more and more prepared and trained to face the complications linked with the virus. 

\paragraph{France.} The French outbreak forecast is presented in Fig.~\ref{fig:sidarthe_fit_fr}. Despite the significant presence of noisy and missing data, we observe that the model succeeds in forecasting the state variables $\t,\h,\e$, always within the tolerance. Instead, the state variable $\r$ is clearly overestimated. We believe it is caused by two main reasons: first, the overestimation of $\d$ overflows to $\r$, and second an abrupt change in the data distribution, since the growth of target data $\hat{\d}(t)$ is not reflected by a similar increase of $\hat{\r}(t)$.
The trend of the model parameters (see Fig.~\ref{fig:sidarthe_params_fr}) is similar to the one obtained for Italy. Recovery rates tend to grow, detection rates quickly increase and then stabilize, symptoms development decreases significantly.


\section{Conclusions}\label{sec:conclusions}
In this paper we have discussed the problem of learning time-variant coefficients in compartmental models, with special attention on SIDARTHE~\cite{Giordano2020}, a recently introduced epidemiological 
model which offers a very rich description of the stages of an epidemic infection.
The major contribution of the paper consists of extending the
challenging features of SIDARTHE model to the case of time-variant
parameters that are properly learned from examples. 
This is carried out within a functional formulation of learning
which is based on a special interpretation of gradient-flow, which
allowed us to obtain a reliable forecasting of most 
critical indicators of the outbreak severity (i.e. deaths, recoveries and hospitalized in ICU individuals) of the COVID-19 epidemic outbreak. A massive experimentation in Italy and France has shown promising results over large windows in the last few months. We are 
confident that the proposed enrichment of SIDARTHE model, which is 
one of top level models for COVID-19 prediction, might be useful
for supporting critical policies to face the diffusion of the 
infection all around the world.

\appendix[Details on algorithmic issues]\label{appendix}
Given a function $u\in X$ and 
the temporal partition $0\equiv t=t_0<t_1<\cdots<t_N\equiv T$,
in this appendix we show how to explicitly construct its 
discrete counterpart as an element of the domain
of the function $f$ and subsequently how to rearrange its 
components to precisely define the quantities $\hat x$ and
$\hat f$ that are used in Eq.~\eqref{eq:reg}.

Let $\cvar_{i,j}:=\cvar_i(t_j)$
the components of the matrix $U\in\R^{18\times(N+1)}$ whose rows are the sampling on the
temporal partition $t_0,t_1,\dots,t_{N}$ of the coordinates of $u$.
Instead of working with matrices we exploit the isomorphism between $\R^{18\times(N+1)}$ and
$\R^{18(N+1)}$ that maps
\[
U\to\vect(U):=(u_{1,0},u_{2,0},\dots,u_{18,0},\dots,u_{1,N},\dots,u_{18,N})'.\]
With this mapping we can transform the initial point $u^0$ of the flow defined by
\eqref{eq:grad-flow-explicit} into the initial point $x^0$ necessary to 
start the gradient descent in Eq.~\eqref{eq:grad-desc}
\footnote{
We adopt the notation $(a_{ij})$ to denote the matrix
whose $ij$-th element is  $a_{ij}$.}: $\var^0=\vect((u^0_{i,j}))\in\R^{18(N+1)}$.

The relation between $x$ and $\hat x_j$ for $j=0,\dots,N$ 
and between $f$ and $\hat f$ instead naturally follows
once we explicitly state the relation between the
the domain of the function $f$ with the 
product space $(\R^{18})^{N+1}:=\prod_{\alpha=1}^{N+1}\R^{18}$.
Indeed the projections $p_j\colon (\R^{18})^{N+1}\to \R^{18}$
map $x\mapsto p_j(x)=(x_{18j+1},\dots, x_{18(j+1)})'=:\hat\var_j$ for 
$j=0,\dots,N$. Following the same line of thoughts it is natural to define
$\hat f\colon(\R^{18})^{N+1}\to[0,+\infty)$ simply as
\[
y\mapsto \hat f(y)\equiv\hat f(y_0,\dots,y_N):= f(c(y_0,\dots, y_N)),
\]
where $c\colon(\R^{18})^{N+1}\to\R^{18(N+1)}$ realizes the isomorphism
\[(y_0,\dots, y_N)\to
({y_0}_1,\dots,{y_0}_{18}, {y_1}_1,\dots,{y_1}_{18},\dots, {y_N}_{18})'.\]
Notice that with this definition $(\nabla \hat f)_i\in\R^{18}$
for all $i=1,\dots, N+1$, while $(\nabla f)_j\in\R$ for all
$j=1,\dots,18(N+1)$.
This being said all quantities used in  Eq.~\eqref{eq:reg}
are precisely defined once we specify that $\hat x$ is used as a 
shortcut for $\hat\var_0,\dots,\hat\var_N$.


\section*{Acknowledgments}
We thank Stefano Merler (FBK) for insightful discussions.

\bibliographystyle{IEEEtran}
\bibliography{covid} 

\end{document}